\documentclass{article}

\usepackage{comment}

\usepackage[main, final]{neurips_2026}

\usepackage[utf8]{inputenc} 
\usepackage[T1]{fontenc}    
\usepackage{hyperref}       
\usepackage{url}            
\usepackage{booktabs}      
\usepackage{amsfonts}       
\usepackage{nicefrac}      
\usepackage{microtype}      
\usepackage{xcolor}         
\usepackage{graphicx}
\usepackage{wrapfig}
\usepackage{subcaption}
\usepackage{caption}
\usepackage{booktabs}
\usepackage{placeins}
\usepackage{multirow}
\usepackage[norelsize,boxed,noend]{algorithm2e}
\usepackage{float}
\SetArgSty{textrm} 
\SetAlCapSkip{0.7em} 
\SetKwInput{KwRequire}{Input}
\SetKwInput{KwEnsure}{Output}

\usepackage{amsmath}
\usepackage{amssymb}
\usepackage{mathtools}
\usepackage{amsthm}
\usepackage{enumitem}
\usepackage{tikz}
\usetikzlibrary{positioning}
\usepackage[capitalize,noabbrev]{cleveref}

\theoremstyle{plain}
\newtheorem{theorem}{Theorem}[section]
\newtheorem{proposition}[theorem]{Proposition}
\newtheorem{lemma}[theorem]{Lemma}

\theoremstyle{definition}
\newtheorem{definition}[theorem]{Definition}

\theoremstyle{remark}

\title{Counterfactual Maps: \\ What They Are and How to Find Them}

\author{%
  Awa Khouna \\
  Polytechnique Montréal\\
  \texttt{awa.khouna@polymtl.ca} \\
  \And 
  Julien Ferry \\
  Polytechnique Montréal\\
  \texttt{julien.ferry@polymtl.ca} \\
  \And 
  Thibaut Vidal \\
  Polytechnique Montréal\\
  \texttt{thibaut.vidal@polymtl.ca} \\
}

\begin{document}
\usetikzlibrary{calc}

\maketitle

\begin{abstract}
    Counterfactual explanations are a central tool in interpretable machine learning, yet computing them exactly for complex models remains challenging. For tree ensembles, predictions are piecewise constant over a large collection of axis-aligned hyperrectangles, implying that an optimal counterfactual for a point corresponds to its projection onto the nearest rectangle with an alternative label under a chosen metric. Existing methods largely overlook this geometric structure, relying either on heuristics with no optimality guarantees or on mixed-integer programming formulations that do not scale to interactive use. 
    In this work, we revisit counterfactual generation through the lens of nearest-region search and introduce \emph{counterfactual maps}, a global representation of recourse. Leveraging the fact that any tree ensemble can be compressed into an equivalent partition of labeled hyperrectangles, we cast counterfactual search as the problem of identifying the \emph{generalized Voronoi cell} associated with the nearest rectangle of an alternative label. This leads to an exact, amortized algorithm based on volumetric k-dimensional (KD) trees, which performs branch-and-bound nearest-region queries with explicit optimality certificates and sublinear average query time after a one-time preprocessing phase.
    Our experimental analyses across several real datasets from high-stakes application domains show that this approach delivers globally optimal counterfactual explanations with millisecond-level latency, achieving query times that are orders of magnitude faster than existing exact, cold-start optimization methods. 
  \end{abstract}

\section{Introduction}

Counterfactual explanations have become a cornerstone of interpretable machine learning, providing actionable insights by identifying how an input must change to alter a model's prediction. Their appeal spans regulated domains such as credit scoring, healthcare, and criminal justice, where explanations must be both meaningful and reliable.
Despite this importance, generating \emph{optimal} counterfactual explanations remains a major challenge for complex models, particularly tree ensembles.

Tree ensembles, including random forests and gradient boosted trees, are widely used for tabular data. They achieve strong predictive performance while offering some interpretability through feature-based splits. Geometrically, these models partition the input space into a collection of axis-aligned hyperrectangles, each with a constant prediction. A globally optimal counterfactual for a given input is therefore the point of minimum distance within any hyperrectangle with an alternative label.

\begin{figure*}[t]
    \centering

    \begin{subfigure}[t]{0.495\textwidth}
    \centering\includegraphics[width=0.875\textwidth]{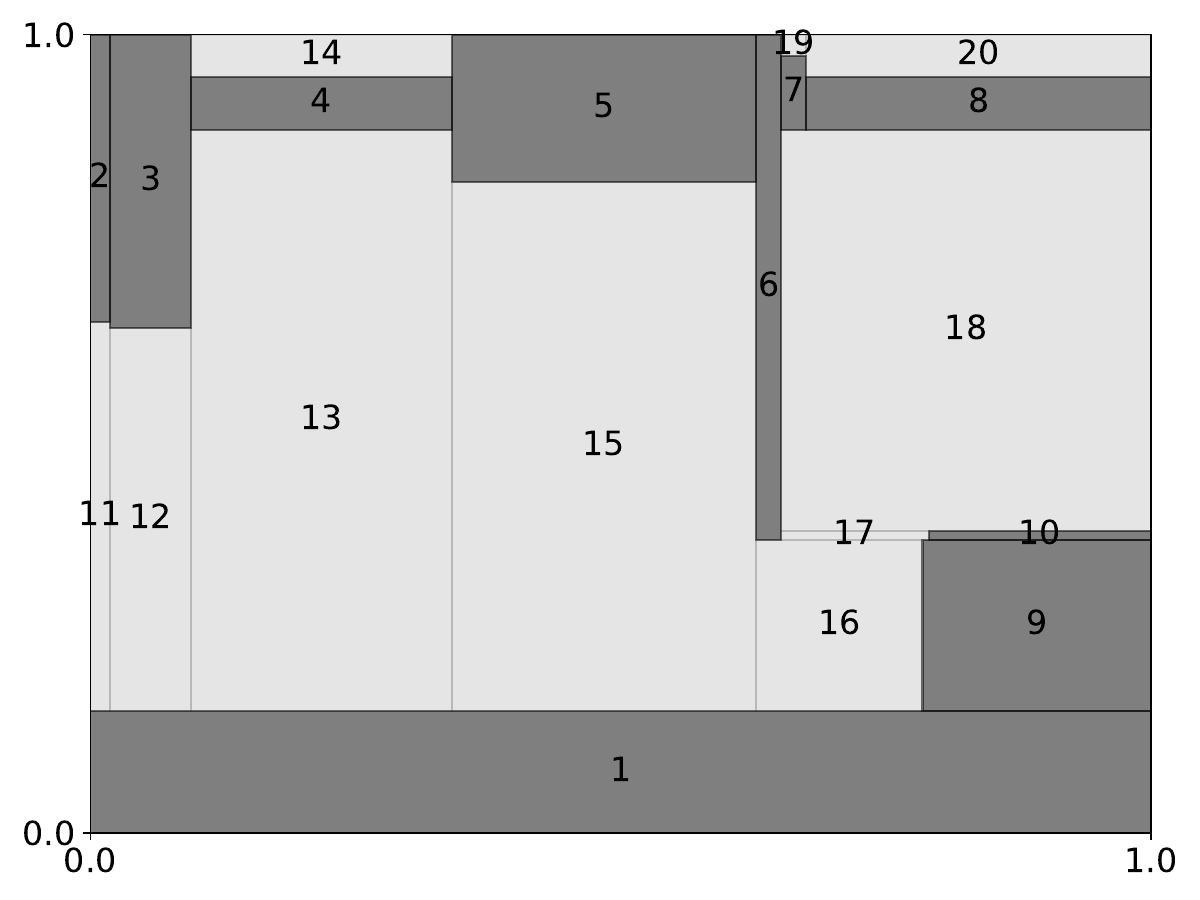}\vspace*{-0.35cm}
    \caption{Decision boundaries of the model.}
   \label{fig:decision_boundary}
  \end{subfigure}
    \hfill
    \begin{subfigure}[t]{0.495\textwidth}
    \centering\includegraphics[width=0.875\textwidth]{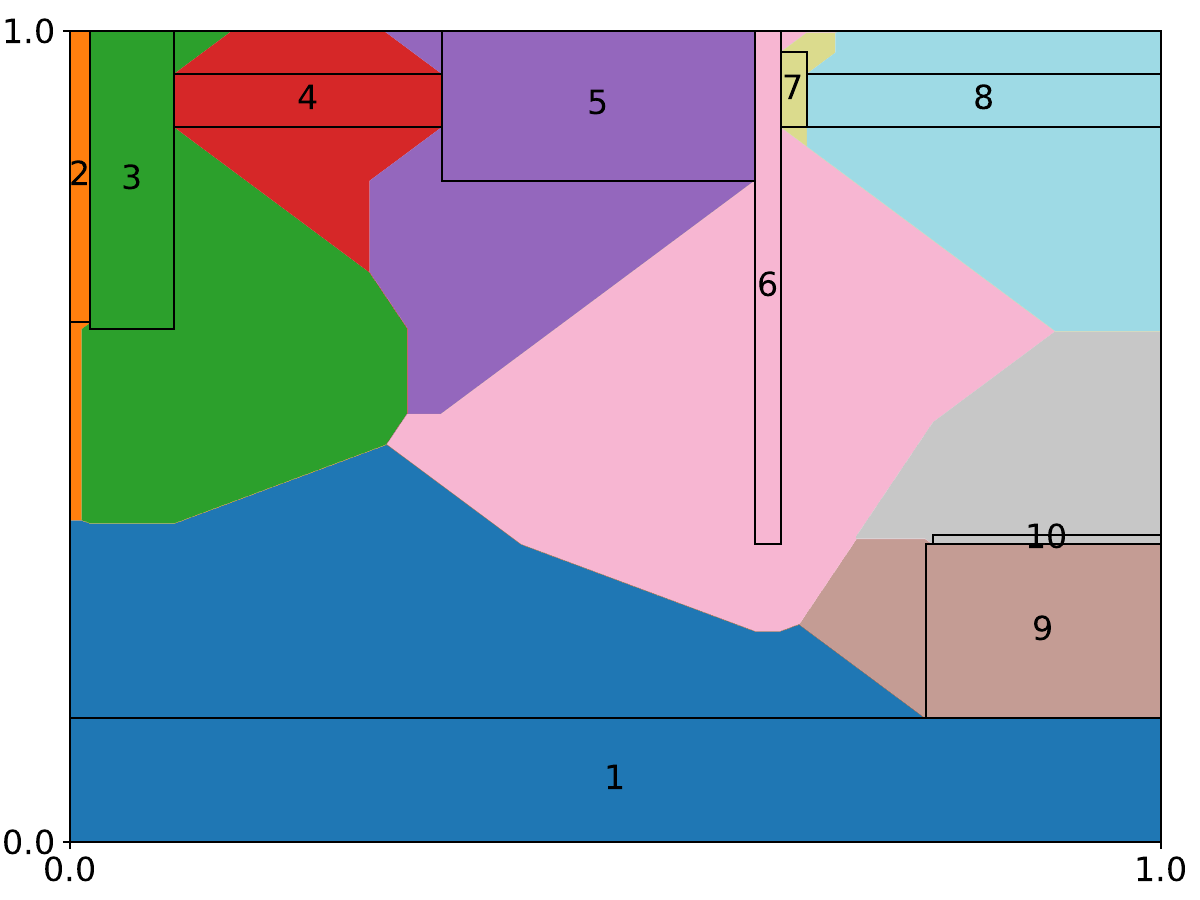}\vspace*{-0.35cm}
  \caption{Counterfactual map for the $L_1$ norm.}
   \label{fig:cf_map}
  \end{subfigure}

    \caption{Counterfactual map for target class \smash{\colorbox[HTML]{7F7F7F}{\textcolor{white}{$y_2$}}}. The map corresponds to the \emph{generalized Voronoi partition} induced by the hyperrectangles of this class. For any input query, it identifies the closest hyperrectangle of this class; projecting the query onto that region yields a globally optimal counterfactual. This illustration uses a random forest with two depth-5 trees trained on the ``blobs'' dataset.}
    \label{fig:illustrative_example}
\end{figure*}

Most existing approaches treat counterfactual generation as a \emph{per-instance} optimization problem.
Heuristic approaches \citep{Tolomei_2017, carreira2023LIRE, Mothilal_2020} traverse decision paths within trees or perturb inputs to find label-changing points, but provide no guarantees of optimality. In practice, these methods may converge to poor local minima, vastly overshoot the minimal required changes, or fail to return a feasible counterfactual even when one exists, exposing \emph{some} of the end users to explanations that may be invalid or unnecessarily costly \citep{parmentier2021optimal}. Exact methods instead rely on mixed-integer programming or satisfiability formulations \citep{parmentier2021optimal, karimi2020modelagnostic}, which provide global optimality certificates at the cost of scalability: solving a mathematical program for each query quickly becomes prohibitive for large ensembles or interactive use. Consequently, no current approach simultaneously offers global optimality and low query latency for realistic tree ensembles.

In this work, we argue that this limitation arises from focusing on individual counterfactuals rather than exploiting the \emph{global structure} of counterfactual explanations through preprocessing. We introduce \emph{counterfactual maps}: global representations that associate each point in the feature space with its nearest decision region of an alternative label, as illustrated in Figure~\ref{fig:illustrative_example}. Such maps enable efficient access to globally optimal counterfactuals at query time without requiring the solution of a new optimization problem.

A key observation underlying counterfactual maps is that any tree ensemble admits an equivalent representation as a decision tree partitioning the feature space into labeled hyperrectangles \citep{vidal2020bornagain}. Under this representation, counterfactual generation reduces to a nearest-region problem. The feature space can therefore be partitioned into \emph{generalized Voronoi cells} induced by these hyperrectangles, where each cell identifies the region yielding the closest counterfactual. Counterfactual maps can be viewed as implicit representations of this Voronoi decomposition.

Building on this perspective, we develop methods to efficiently construct and query counterfactual maps using a two-stage preprocess-and-query paradigm. In a one-time preprocessing step, we extract a labeled hyperrectangular partition equivalent to the ensemble's decision function and index it using geometric search structures. At query time, this index identifies the nearest opposing region and produces a globally optimal counterfactual by projection. We instantiate this framework using orthogonal range searching and volumetric KD-trees, which provide explicit optimality certificates and achieve sublinear average query time. Our experiments on synthetic and real-world tabular datasets demonstrate that counterfactual maps enable globally optimal counterfactual explanations with interactive latency, yielding fast, guaranteed, and scalable counterfactuals for tree ensembles.

\section{Counterfactual Maps}
\label{sec:count-maps}

We consider classifiers of the form $\mathcal{T}:\mathcal{X}\to\mathcal{Y}$, where the input domain is a product space 
\mbox{$\mathcal{X}=\mathcal{X}_1\times\cdots\times\mathcal{X}_m$}, with mixed continuous, categorical, and binary features, and the output space consists of $c$ classes, \mbox{$\mathcal{Y}=\{y_1,\ldots,y_c\}$}.

\textbf{Tree ensembles and induced partitions.}
A tree ensemble (random forest, gradient-boosted trees) induces a piecewise-constant decision function. Each decision tree partitions $\mathcal{X}$ into hyperrectangles associated with leaf predictions. Aggregating all trees yields a finite collection of \emph{regions} over which the ensemble prediction is constant. These regions can be represented as a finite set of disjoint axis-aligned hyperrectangles \citep{vidal2020bornagain}:
$$
\mathbb{H} \triangleq \{\mathcal{H}_1,\ldots,\mathcal{H}_n\}, \ \
\bigcup_{i=1}^n \mathcal{H}_i = \mathcal{X}, \ \
\mathcal{H}_i\cap\mathcal{H}_j=\emptyset \ \ \forall i \neq j.
$$
Each region $\mathcal{H}_i$ is associated with a unique class label $\smash{\mathcal{T}(\mathcal{H}_i)}$. In the worst case, the number of regions $n$ grows exponentially with the depth of the ensemble.

\textbf{Optimal counterfactuals.}
Let $d_p(\cdot,\cdot)$ denote an $L_p$ distance on $\mathcal{X}$. For any point $x\in\mathcal{X}$ and hyperrectangle $\mathcal{H}\subseteq\mathcal{X}$, we define the point--rectangle distance
\[
d_p(x,\mathcal{H}) \triangleq \inf_{x'\in\mathcal{H}}\|x - x'\|_p,
\]
which admits a closed-form, coordinate-wise expression for all $p \geq 1$. Let $\Pi_{\mathcal{H}}(x)$ denote a projection of $x$ onto $\mathcal{H}$ attaining this minimum. More generally, we allow feature-wise weights $w\in\mathbb{R}_{+}^m$ to reflect different actionability costs and use the weighted norm $\|u\|_{p,w} \smash{\triangleq}\|\mathrm{diag}(w)\,u\|_p$, equivalently rescaling coordinates by $\tilde x=\mathrm{diag}(w)x$. For readability, we write $d_p$ and understand it as either the standard or weighted $L_p$ distance.

For an input $x$ with prediction $y=\mathcal{T}(x)$ and a target label $y'\neq y$, the set of regions with label $y'$ is:
\[
\mathbb{H}_{y'} \triangleq \{\mathcal{H}\in\mathbb{H}\mid \mathcal{T}(\mathcal{H})=y'\}.
\]

A globally optimal counterfactual targeting $y'$ is obtained by projecting $x$ onto any minimizer
\[
\mathcal{H}^\star \in \arg\min_{\mathcal{H}\in\mathbb{H}_{y'}} d_p(x,\mathcal{H}),
\qquad
x_{\mathrm{cf}} \triangleq \Pi_{\mathcal{H}^\star}(x).
\]

\begin{definition}[Counterfactual map]
Given a tree ensemble~$\mathcal{T}$ inducing a partition $\mathbb{H}$ of the feature space, a distance~$d_p$, and a target label~$y'$, a \emph{counterfactual map} is a mapping
$\smash{\hat{f}_{y'} : \mathcal{X} \to \mathbb{H}_{y'}}$
such that, for every $x\in\mathcal{X}$,
$
\smash{\hat{f}_{y'}(x)} \in \arg\min_{\mathcal{H}\in\mathbb{H}_{y'}} d_p(x,\mathcal{H}).
$
\end{definition}

The associated counterfactual explanation is then obtained by projection:
$\smash{x_{\mathrm{cf}} = \Pi_{\hat{f}_{y'}(x)}(x)}$.
Note that the minimizer in Definition~2.1 may not be unique (notably for $p\in\{1,\infty\}$). Any selection from the argmin set yields a globally optimal counterfactual via projection.

\textbf{Problem statement.}
Given a tree ensemble $\mathcal{T}$, a distance $d_p$, and a target label $y'$, our objective is to construct a counterfactual map (or an equivalent data structure) that, for any query point $x$, identifies the nearest region of class $y'$. Subsequently, the counterfactual is obtained by projecting onto this region. This paper focuses on efficiently constructing and querying such counterfactual maps.

\textbf{Structural properties.}
Counterfactual maps admit a geometric interpretation as generalized Voronoi diagrams of the feature space induced by the hyperrectangles in $\mathbb{H}_{y'}$. The geometry of the resulting cells depends on the choice of the distance $d_p$, as formalized below.

\begin{proposition}
\label{prop:cf_maps_structure}
For $p\in[1,\infty]$ and $\mathcal{H}_i\in \mathbb{H}_{y'}$, define the (closed) $p$-Voronoi region
$\smash{\smash{\mathcal{V}^{(p)}_i}} \smash{\triangleq} \{x\in\mathcal{X}:\ d_p(x,\mathcal{H}_i)\le d_p(x,\mathcal{H}_j)\ \forall\,\mathcal{H}_j\in\mathbb{H}_{y'}\},$
and for $\mathcal{H}_i,\mathcal{H}_j\in\mathbb{H}_{y'}$ the bisector $\smash{B^{(p)}_{ij}} \smash{\triangleq} \{x\in\mathcal{X}:\ d_p(x,\mathcal{H}_i)=d_p(x,\mathcal{H}_j)\}$.

$\bullet$ For $p\in\{1,\infty\}$ and any $\mathcal{H}_i,\mathcal{H}_j\in\mathbb{H}_{y'}$, the bisector $\smash{B^{(p)}_{ij}}$ is a finite union of polyhedra. Consequently, each region $\smash{\mathcal{V}^{(p)}_i}$ is a finite union of polyhedra.

$\bullet$ For $p=2$ and any $\mathcal{H}_i,\mathcal{H}_j\in\mathbb{H}_{y'}$, the bisector $\smash{B^{(2)}_{ij}}$ is contained in a finite union of quadratic hypersurfaces; in particular, Voronoi regions are not polyhedral in general.
\end{proposition}

A proof of Proposition~\ref{prop:cf_maps_structure} is given in Appendix~\ref{structure_proof}. Figure~\ref{fig:illustrative_example_extended} in Appendix~\ref{app:illustrative_example_extended} illustrates the different counterfactual map boundaries under the $L_1$, $L_2$ and $L_\infty$ norms. Given the observed nonlinearities in general cases, we opted to represent counterfactual maps implicitly via certified nearest-region search rather than explicitly constructing the Voronoi decomposition. Notably, the KD-tree branch-and-bound procedure introduced in Section~\ref{sec:Metho} applies generally to any $L_p$ distance.

\section{Methodology}
\label{sec:Metho}
\begin{figure*}[!t]
\centering
\resizebox{0.9\linewidth}{!}{%
\begin{tikzpicture}[
    font=\footnotesize,
    node distance=10mm and 14mm,
    stage/.style={rectangle, rounded corners, draw, very thick, align=left, inner sep=8pt, fill=gray!5},
    edge/.style={->, very thick},
    tinytitle/.style={font=\scriptsize\bfseries, text=black!70},
    note/.style={font=\footnotesize, text=black!60},
    panelRed/.style={rectangle, rounded corners, draw, very thick, fill=red!5, inner sep=8mm},
    panelBlue/.style={rectangle, rounded corners, draw, very thick, fill=blue!5, inner sep=8mm}
]

\def\PanelW{198mm}
\def\PanelH{32mm}
\def\VGap{2mm}

\node[panelRed, minimum width=\PanelW, minimum height=\PanelH] (build) {};
\node[font=\bfseries] at ($(build.north)+(0,-4mm)$) {Build pipeline (Algorithm~\ref{alg:build})};
\node[note] at ($(build.north)+(0,-8mm)$) {(one-time)};

\coordinate (bwest) at ($(build.west)+(5mm,-4mm)$);
\coordinate (inputT) at (bwest);

\node[stage, right=25mm of inputT, minimum width=52mm] (ba) {\textbf{Rectangles extraction}\\
BA-Trees~\citep{vidal2020bornagain}\vspace{2pt}\\
$\mathbb{H}=\{\mathcal{H}_1,\dots,\mathcal{H}_n\}$ disjoint, labelled};

\node[stage, right=23mm of ba, minimum width=50mm] (exact) {%
\textbf{Exact search structures}\vspace{2pt}\\
(KD-trees)};

\coordinate (outL) at ($(exact.east)+(32mm,0)$);

\draw[edge] (inputT) -- node[above]{\textbf{Tree ensemble}}
      node[below]{$\mathcal T : \mathcal X \to \mathcal Y$} (ba);
\draw[edge] (ba) -- node[above]{Rectangles} node[below]{$\mathbb{H}$} (exact);
\draw[edge] (exact.east) -- node[above,align=center]{\textbf{Counterfactual Maps}}
      node[below,align=center]{$\hat f_{y'}: \mathcal{X}\to \mathbb{H}$ \\ for each target class $y'$} (outL);

\node[panelBlue, minimum width=\PanelW, minimum height=\PanelH - 5mm ,
      below=\VGap of build] (query) {};
\node[font=\bfseries] at ($(query.north)+(0,-4mm)$) {Query pipeline (Algorithm~\ref{alg:query})};

\coordinate (qwest) at ($(query.west)+(5mm,-4mm)$);
\coordinate (xinp) at (qwest);

\node[stage, right=34mm of xinp, minimum width=48mm] (nearest) {\textbf{Nearest rectangle}\\
$\mathcal{H}^\star=\hat f_{y'}( x)$};

\node[stage, right=23mm of nearest, minimum width=48mm] (clamp) {\textbf{Projection}\\
$ x_{\mathrm{cf}}=\Pi_{\mathcal{H}^\star}( x)$};

\coordinate (outcf) at ($(clamp.east)+(34mm,0)$);

\draw[edge] (xinp) -- node[above, align=center]{\textbf{Query}, \\ \textbf{Target class}}
      node[below, align=center]{$ x\in\mathcal{X}$, \\ $y' \in \mathcal{Y}$} (nearest);
\draw[edge] (nearest) -- node[above]{$\mathcal{H}^\star$} (clamp);
\draw[edge] (clamp) -- node[above]{\textbf{Counterfactual}}
      node[below]{$x_{\mathrm{cf}}$} (outcf);

\end{tikzpicture}%
}
\caption{Two-stage workflow for counterfactual generation in tree-ensemble decision spaces.
The \emph{Build pipeline} (one-time) extracts the labeled rectangular partition $\mathbb{H}$
and constructs an exact counterfactual map using geometric search structures.
The \emph{Query pipeline} identifies the nearest opposing region and projects the query to
produce a globally optimal counterfactual explanation.}
\label{fig:pipeline_build_query}
\end{figure*}
We introduce a two-stage preprocess-and-query pipeline, illustrated in Figure~\ref{fig:pipeline_build_query}. The preprocessing stage is executed once per model and constructs the counterfactual maps, i.e., geometric indices over the ensemble decision regions, one for each possible target class. Subsequently, the query stage uses the appropriate map to answer targeted counterfactual requests with low latency by locating and projecting onto the nearest region of the desired class.

\paragraph{Constructing the Counterfactual Map.}

Constructing a counterfactual map requires a geometric representation of the ensemble's decision function that is both \emph{complete}, i.e., covering the entire feature space, and \emph{faithful}, in the sense that it preserves the original class assignment on every region. Violations of either property would lead to counterfactual explanations that are incorrect or suboptimal.
Any tree ensemble admits such a representation as a partition of the feature space into disjoint, labeled axis-aligned hyperrectangles, as defined in
Section~\ref{sec:count-maps}. 
While a naïve construction could be obtained by intersecting all split thresholds across the ensemble (giving an intractable, exponentially large partition), we rely instead on the born-again tree framework of \citet{vidal2020bornagain} (BA-Trees) as an extraction mechanism, since it produces a compact region-based representation that is globally faithful to the ensemble's predictions, i.e., it encodes the same decision function.
That work proposes two exact algorithms: a dynamic programming approach that yields a minimal partition in terms of the number of hyperrectangles, at the cost of substantial computational and memory overhead; and a faster procedure that preserves faithfulness while relaxing minimality. As minimality has only a moderate impact on query-time complexity, we rely on the latter to limit preprocessing cost.

The resulting labeled hyperrectangles are then indexed to form the counterfactual maps. For each target class $y'$, we build a separate volumetric KD-tree over the subset $\mathbb{H}_{y'}$, yielding a hierarchical spatial index that supports efficient nearest-region queries. This preprocessing is performed once per model and amortized across all subsequent counterfactual queries, in contrast to exact counterfactual methods based on mixed-integer programming, which solve a global optimization problem for each query. Algorithm~\ref{alg:build} summarizes the preprocessing pipeline.

\begin{center}
\begin{minipage}{0.95\textwidth}
\begin{algorithm}[H]
\small
\caption{Preprocessing the Counterfactual Maps}
\label{alg:build}
\KwRequire{Trained tree ensemble $\mathcal{T}$}
\KwEnsure{One KD-tree per target class $y'$}

Extract the rectangular partition $\mathbb{H}$ from $\mathcal{T}$ using BA-Trees~\citep[section 4.5]{vidal2020bornagain}\;
For each $\mathcal{H}_i \in \mathbb{H}$, store its bounds and class label\;

\For{each class $y' \in \mathcal{Y}$}{
    Build a volumetric KD-tree whose leaves are the rectangles in $\mathbb{H}_{y'}$, with each internal node storing the bounding box of all descendant leaves\;
}
\end{algorithm}
\end{minipage}
\end{center}

\paragraph{Indexing the Map for Dynamic Queries.}
Once constructed, the counterfactual maps act as geometric oracles to support exact nearest-region queries. Given a query point $x$ and a target label $y'$, the task is to identify the region in $\mathbb{H}_{y'}$ minimizing the point-rectangle distance $d_p$, and to return the corresponding counterfactual by projection.
A naive evaluation of all regions in $\mathbb{H}_{y'}$ would require linear time in the size of the partition and is therefore impractical. Instead, we perform certified nearest-region search using a branch-and-bound traversal of the KD-tree associated with the target class $y'$. This procedure relies on admissible lower bounds on the distance from the query point to all hyperrectangles contained in a subtree, allowing entire subtrees to be pruned without sacrificing optimality. In practice, this yields sublinear average query time, as only a small fraction of nodes is typically visited.

\begin{center}
\begin{minipage}{0.95\textwidth}
\begin{algorithm}[H]
\small
\caption{Exact Counterfactual Query (Targeted)}
\label{alg:query}
\KwRequire{Query point $x$, KD-tree over $\mathbb{H}_{y'}$, distance $d_p$}
\KwEnsure{Optimal counterfactual $x_{\mathrm{cf}}$}

Initialize best distance, region \& min-priority queue 
$(d^\star,\mathcal{H}^\star,Q) \gets (\infty,\varnothing,\{\mathrm{root}\})$\;

\While{$Q \neq \varnothing$}{
    Pop a node $v$ from $Q$ with smallest lower bound $d_p(x,B(v))$\;

    \If{$d_p(x,B(v)) \ge d^\star$}{
        \textbf{break}\;
    }

    \If{$v$ is a leaf}{
        $(d^\star,\mathcal{H}^\star) \gets 
        \min\bigl((d^\star,\mathcal{H}^\star),
        \{(d_p(x,\mathcal{H}),\mathcal{H}) : \mathcal{H}\in\mathbb{H}_{y'}(v)\}\bigr)$\;
    }
    \Else{
        $Q \gets Q \cup \{c : c\in\mathrm{children}(v),\ d_p(x,B(c))<d^\star\}$\;
    }
}

\Return{$\Pi_{\mathcal{H}^\star}(x)$}\;
\end{algorithm}
\end{minipage}
\end{center}

Algorithm~\ref{alg:query} details the query procedure.  It maintains a priority queue of KD-tree nodes ordered by their lower bounds. At each iteration, the node with the smallest bound is explored; the search terminates whenever this bound exceeds the best distance found so far. When a leaf node is reached, distances to the contained hyperrectangles of class $y'$ are evaluated explicitly, and the current best solution is updated. Once the nearest target region $\mathcal{H}^\star$ is identified, the counterfactual is obtained by projection, $x_{\mathrm{cf}}=\Pi_{\mathcal{H}^\star}(x)$.
The following result establishes the correctness of the~procedure.

\begin{theorem}[Exact Nearest-Hyperrectangle Search]
\label{thm:kd_exact}
Let $\mathbb{H}_{y'}$ be the set of axis-aligned hyperrectangles associated with a target label $y'$, and let $d_p$ be an $L_p$ distance with $1\le p\le\infty$. Given a query $x\in\mathbb{R}^m$, Algorithm~\ref{alg:query} returns a hyperrectangle $
\mathcal{H}^\star \in \smash{\arg\min_{\mathcal{H}\in\mathbb{H}_{y'}}} d_p(x,\mathcal{H})$, and the projection $x_{\mathrm{cf}}=\Pi_{\mathcal{H}^\star}(x)$ is a globally optimal counterfactual explanation targeting class $y'$ under $d_p$.
\end{theorem}

A detailed proof of Theorem~\ref{thm:kd_exact} is provided in Appendix~\ref{kd_exact_proof}. It relies on the fact that the branch-and-bound traversal maintains valid lower bounds on the distance to all unexplored regions, ensuring that no region that could improve the current best solution is ever pruned. Termination occurs only when no unexplored subtree can contain a closer target region, guaranteeing global optimality.

Finally, CF-Maps can handle additional desiderata on counterfactuals, such as actionability (e.g., structural constraints, monotonicity, or immutability of certain features) and plausibility (closeness to the data manifold) at two complementary stages:
(i) during preprocessing, by excluding regions that violate those constraints or are implausible with respect to the data manifold, and 
(ii) at query time, by using a user-specified distance function in the branch-and-bound traversal to encode user-specific actionability costs, monotonicity preferences, or prohibitive penalties for immutable features, without rebuilding the KD-trees.

\section{Experimental Analyses}\label{sec:experiments}

We evaluate the proposed counterfactual maps algorithm (\textbf{CF-Maps}) on a diverse set of tabular classification benchmarks and compare it against state-of-the-art exact and heuristic counterfactual methods for random forests. Our experiments address two main questions: (i) to what extent can preprocessing amortize the cost of exact counterfactuals to achieve interactive query latency, and (ii) how does CF-Maps compare to existing exact and heuristic baselines in terms of runtime and counterfactual quality.

\subsection{Experimental setup}\label{subsec:exp_setup}

\textbf{Datasets.} We consider five tabular datasets, summarized in Table~\ref{tab:full_datasets} (Appendix \ref{app:datasets_detailed}), covering numerical, categorical, ordinal, and binary features. Four of them are commonly used in trustworthy machine learning and are associated with high-stakes tasks: recidivism prediction (CP), credit risk assessment (FI), and medical prognosis or disease prevalence prediction (BC and PD). In these settings, system audits, explanation quality, and reliable recourse are particularly important. We also include the classical Seeds dataset (SE) to evaluate the method in a multiclass setting. To ensure consistency and reproducibility, we follow the preprocessing pipeline of \citet{vidal2020bornagain} for feature scaling and categorical-variable encoding.

\textbf{Models.} On each dataset, we train random forests using scikit-learn~\citep{scikit-learn}, varying the number of trees (5, 10, 20, 50, or 100) and their maximum depth (2, 3, 4, 5, 6, or~7). We repeat all experiments over five random seeds. Each dataset is split into $80\%$ for training and $20\%$ for testing. For each trained forest, we generate counterfactual explanations for a total of $1000$ query points, drawn uniformly from the test set and supplemented, if necessary, with uniformly sampled inputs.

\textbf{Baselines.}
We compare the proposed \textbf{CF-Maps} against five baselines covering both exact and heuristic counterfactual generation for random forests: \textbf{OCEAN} \citep{parmentier2021optimal} is the current state-of-the-art for generating optimal counterfactual explanations for tree ensembles, using mixed-integer optimization. We use the reference implementation \texttt{oceanpy} from PyPI\footnote{\url{https://pypi.org/project/oceanpy/}} with its default solver settings. \textbf{Feature Tweaking (FT)} \citep{Tolomei_2017} is a tree-specific heuristic that perturbs decision paths to flip predictions. We use the public implementation from the repository \texttt{featureTweakPy}\footnote{\url{https://github.com/upura/featureTweakPy}}. \textbf{LIRE} \citep{carreira2023LIRE} is a heuristic method that restricts the search space to regions supported by the training data, also known as \emph{live regions}. Since, to our knowledge, no official implementation is available, we reproduced LIRE as described in its paper. \textbf{EECE}~\citep{zhang2025efficient}\footnote{\url{https://github.com/Haifei-ZHANG/EECE}} is a recent random-forest-specific method inspired by \textbf{FT}. It generates candidate counterfactuals from individual tree leaves using an improved tree representation and search strategy, and incorporates \textbf{LIRE}'s active regions as additional candidates to ensure that a valid counterfactual can be found. \textbf{DiCE} \citep{Mothilal_2020} is a gradient-free, sampling-based method designed to generate diverse counterfactual explanations. Unlike tree-specific methods, it does not exploit the ensemble's internal combinatorial structure and therefore serves as a generic baseline. We use the public implementation in the \texttt{dice-ml}\footnote{\url{https://pypi.org/project/dice-ml/}} package.

We further discuss these baselines and position our approach with respect to related work in Section~\ref{sec:related_works}.

\textbf{Metrics.} Since CF-Maps and OCEAN return globally optimal counterfactuals by construction, standard quality metrics such as validity and distance are not discriminative between the two methods: validity is always satisfied, and distance equals the optimum under the chosen norm.
We therefore focus on preprocessing time, query latency, and the validity and distance suboptimality of heuristic baselines.  
We also provide qualitative counterfactual examples in Appendix~\ref{sec:examples_cfs}.

\textbf{Setup.} All experiments are run on a computing cluster over a set of homogeneous nodes equipped with AMD EPYC 7532 (Zen 2) @2.40GHz CPU with 64GB of RAM. All materials (source code and datasets) required to reproduce our experiments and figures are accessible at \url{https://github.com/vidalt/CF-Maps}.

\subsection{Experimental Results}
\label{sec:results}

\begin{wrapfigure}{r}{0.5\textwidth}
    \vspace{-1cm}
    \centering
    \includegraphics[width=0.95\linewidth]{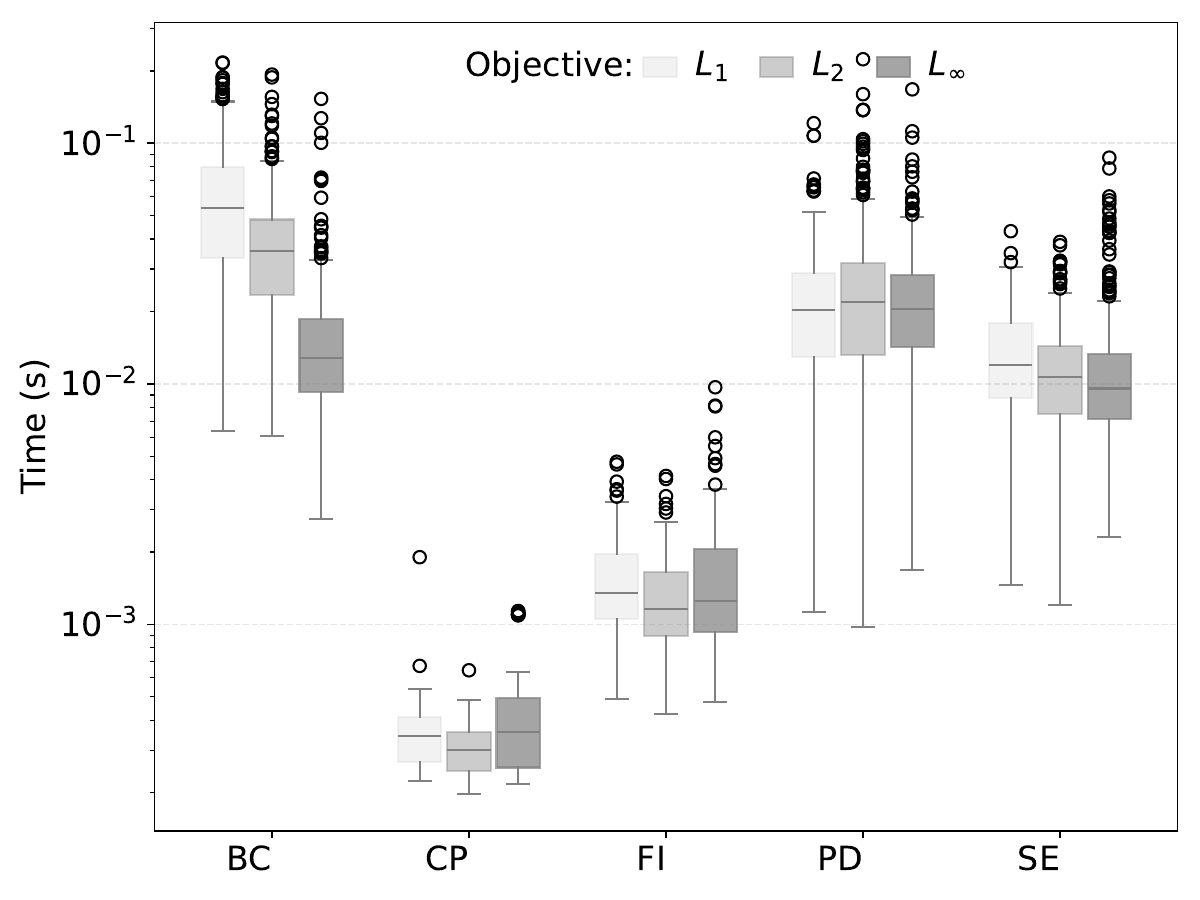}
    \caption{Average query latency of \textbf{CF-Maps} after the one-time preprocessing step (i.e., time to generate a counterfactual explanation using Algorithm~\ref{alg:query}), for three distance functions across all datasets, for random forests with 100 depth-5 trees.}
    \label{fig:boxplot_all_norms}
    \vspace{-0.5cm}
\end{wrapfigure}

\textbf{Result 1. CF-Maps achieves millisecond-to-subsecond query times suitable for interactive use while delivering optimal counterfactuals.} 
Figure~\ref{fig:boxplot_all_norms} reports the average time required by CF-Maps to generate each counterfactual explanation for a standard random forest configuration of 100 depth-5 trees, averaged over five random seeds and 1000 queries, after the one-time preprocessing step. 
Across datasets, CF-Maps returns optimal counterfactuals within fractions of a second, ranging from a few milliseconds on CP to a few tens of milliseconds on PD. Variability across datasets is primarily driven by dimensionality (number of features), but also by feature type: numerical attributes typically induce many distinct split thresholds, leading to a partition with more hyperrectangles in $\mathbb{H}$ and a deeper volumetric KD-tree.

\textbf{Result 2. CF-Maps adapts effectively to different norms.} Figure~\ref{fig:boxplot_all_norms} demonstrates that CF-Maps can efficiently handle (and therefore provide optimal solutions for) all the considered distances ($L_1$, $L_2$, and $L_\infty$) while maintaining interactive query latency. Query times are consistently slightly higher under $L_1$ and lower under $L_\infty$, with $L_2$ in between. This reflects the pruning efficiency of the KD-tree branch-and-bound search: for axis-aligned hyperrectangles, distance lower bounds tend to be tighter under $L_\infty$.

\begin{table*}[t!]
\centering
\caption{Preprocessing time ($T_P$) and time to generate 1000 counterfactuals ($T_{1000}$), both in seconds, for random forests with 100 depth-5 trees. We report the proportion (\%F) of queries failing to find a valid counterfactual, and the distance ratio $D$ relative to the optimal distance (computed over valid counterfactuals only).
Best results from optimal counterfactual methods are highlighted in \textbf{bold}.}\label{tab:numcf_vs_time_table}
\renewcommand{\arraystretch}{1.12}
\setlength{\tabcolsep}{3pt}
\resizebox{\textwidth}{!}{
\begin{tabular}{lrrcr@{\hspace*{0.2cm}}|@{\hspace*{0.2cm}}rrcr@{\hspace*{0.2cm}}||@{\hspace*{0.2cm}}rrcr@{\hspace*{0.2cm}}|@{\hspace*{0.2cm}}rrr@{\hspace*{0.2cm}}|@{\hspace*{0.2cm}}rrr|@{\hspace*{0.2cm}}rrcr}
\toprule
\textbf{Dataset} \hspace*{-0.3cm} & \multicolumn{4}{c}{OCEAN} & \multicolumn{4}{c}{CF-Maps} & \multicolumn{4}{c}{LIRE} & \multicolumn{3}{c}{FT} & \multicolumn{3}{c}{DiCE} & \multicolumn{4}{c}{EECE} \\
 & $T_P$ & $T_{1000}$ & \%F & $D$ & $T_P$ & $T_{1000}$ & \%F & $D$ & $T_P$ & $T_{1000}$ & \%F & $D$ & $T_{1000}$ & \%F & $D$ & $T_{1000}$ & \%F & $D$ & $T_P$ & $T_{1000}$ & \%F & $D$ \\
\midrule
BC & 3.0 & 1143 & 0 & 1.0 & 557 & \textbf{45.2} & 0 & 1.0 & 3.9 & 1.3 & 0 & 2.3 & 107 & 93.2 & 1.6 & 1951 & 2.2 & 2.8 & 4.0 & 19.8 & 0 & 1.2 \\
CP & 5.0 & 619 & 0 & 1.0 & 0.2 & \textbf{1.4} & 0 & 1.0 & 2.0 & 1.2 & 0 & 1.1 & 239 & 41.4 & 1.0 & 1483 & 15.1 & 1.0 & 2.4 & 67.9 & 0 & 1.0 \\
FI & 5.9 & 1273 & 0 & 1.0 & 0.9 & \textbf{17.6} & 0 & 1.0 & 33.0 & 23.6 & 0 & 1.3 & 732 & 63.0 & 1.0 & 1885 & 16.6 & 1.3 & 33.6 & 64.1 & 0 & 1.3\\
PD & 13.9 & 42983 & 0 & 1.0 & 467 & \textbf{32.1} & 0 & 1.0 & 15.9 & 2.4 & 0 & 2.4 & 178 & 79.3 & 3.2 & 1278 & 1.0 & 4.4 & 16.1 & 28.7 & 0 & 1.1  \\
SE & 2.9 & 5069 & 0 & 1.0 & 120 & \textbf{8.5} & 0 & 1.0 & 2.2 & 0.5 & 0 & 3.4 & 55.6 & 78.5 & 1.9 & 866 & 1.0 & 4.8 & 2.2 & 17.9 & 0 & 2.3  \\
\bottomrule
\end{tabular}
}\vspace*{0.35cm}
\end{table*}
\begin{figure*}[t!]
    \centering
    \includegraphics[width=\linewidth]{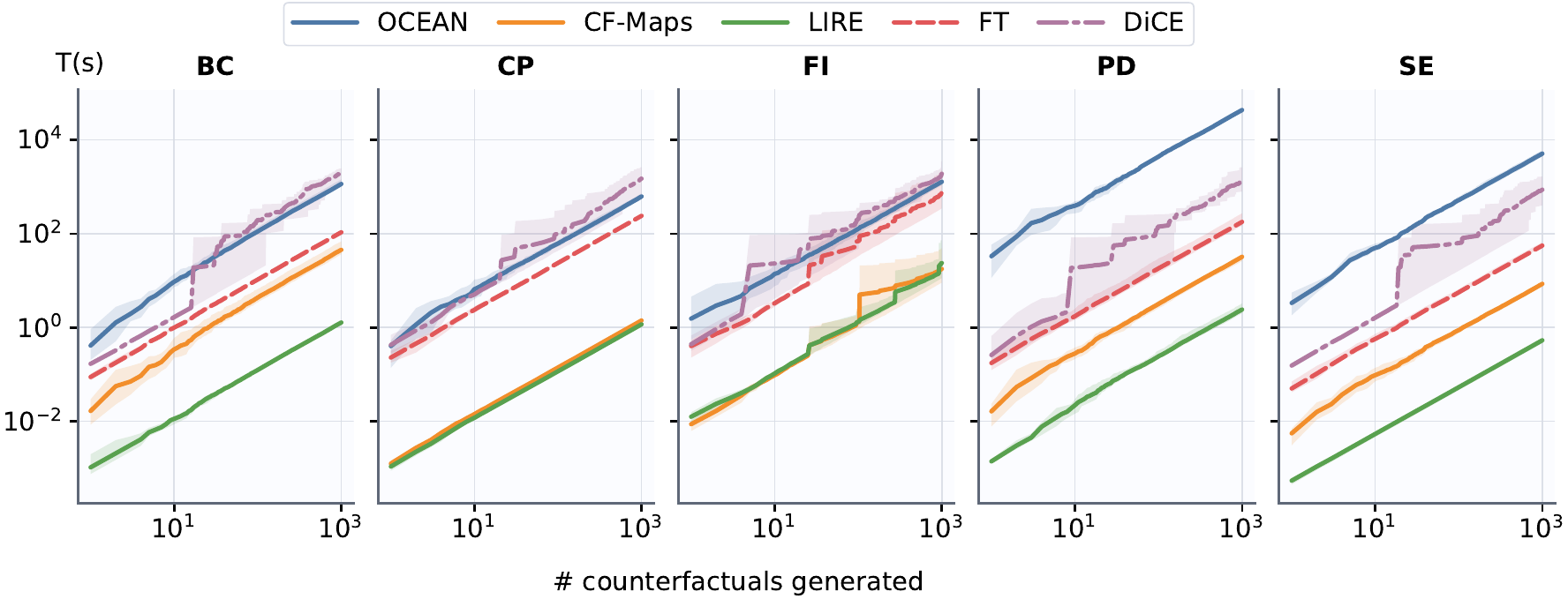}
\caption{Cumulative query time as a function of the number of generated counterfactuals, for random forests with 100 depth-5 trees. The shaded area shows the range between the fastest and slowest runs over the five random seeds.} \label{fig:numcf_vs_time}
\vspace{-0.25cm}
\end{figure*}

\textbf{Result 3. CF-Maps is orders of magnitude faster than the state-of-the-art for generating optimal counterfactual explanations, while heuristic baselines are suboptimal and often slower.}
Table~\ref{tab:numcf_vs_time_table} compares CF-Maps with existing exact and heuristic baselines across all datasets for random forests with 100 depth-5 trees. We report the preprocessing time $T_P$, generation time for 1000 counterfactuals $T_{1000}$, and the suboptimality (distance ratio $D$ relative to the optimum) and validity (percentage of failures \%F) of the methods. Note that optimal methods such as OCEAN and CF-Maps always guarantee validity (\%F $=0$) and optimality ($D=1.0$). LIRE and EECE also guarantee validity by design. Additional details on the KD-Trees produced during CF-Maps preprocessing, including their size and average number of nodes visited per query, are reported in Table~\ref{tab:kd-tree} in Appendix~\ref{sec:app_kdtree}.

As shown in the table, heuristic methods incur substantially higher recourse costs, with average distance ratios reaching up to $3.4\times$ for LIRE, $3.2\times$ for FT, $4.8\times$ for DiCE, and $2.3\times$ for EECE relative to the optimum. Moreover, the distance ratios of DiCE and FT are computed only over queries for which a valid counterfactual is found; for the realistic-size random forests considered here, FT fails to return any valid counterfactual for most queries (from roughly 41\% on CP to 93\% on BC), and DiCE also fails on a non-negligible subset of cases, with failure rates up to 16.6\% on FI. In contrast, OCEAN consistently returns optimal counterfactuals, but its runtime grows markedly with the number of queries, as each explanation requires solving a new optimization problem. 

CF-Maps avoids both limitations: it delivers globally optimal counterfactuals while achieving query latency comparable to LIRE and EECE and orders of magnitude lower than OCEAN (and often DiCE). Figure~\ref{fig:numcf_vs_time} further illustrates the same trend. 

\begin{figure*}[t!]
    \centering
    \includegraphics[width=\linewidth]{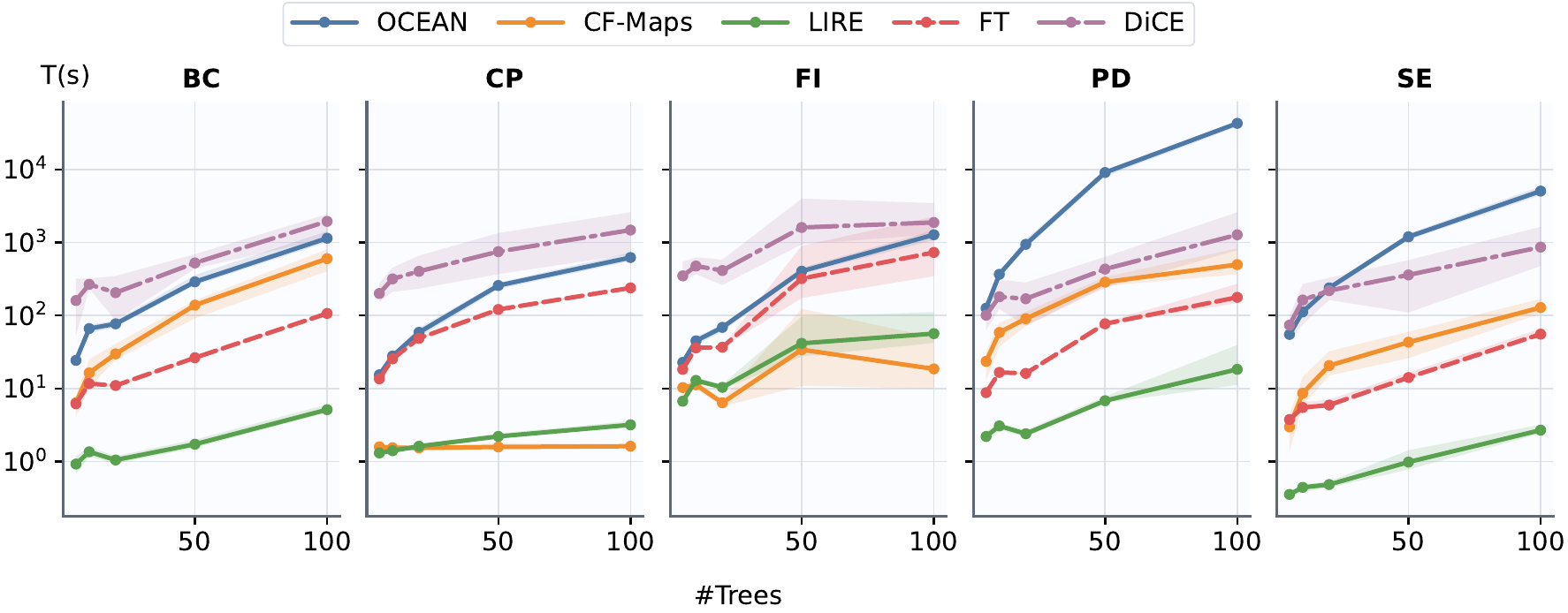}
\caption{Total time (including preprocessing) for generating 1000 counterfactuals, for random forests with varying numbers of depth-5 trees. The shaded area shows the range between the fastest and slowest runs over the five random seeds} \label{fig:time_vs_estimators}
\vspace{-0.3cm}
\end{figure*}

\textbf{Result 4. Increasing ensemble complexity mainly affects CF-Maps preprocessing, while query times remain low.}
Figure~\ref{fig:time_vs_estimators} shows that increasing the number of depth-5 trees raises the cost of partition extraction, while CF-Maps queries remain fast. Similar trends hold when increasing tree depth (Figure~\ref{fig:time_vs_depth} in Appendix~\ref{app:complementary_results}).
These results highlight an important distinction between the one-time cost of constructing a counterfactual map and the cost of querying it. The baselines exhibit different trade-offs. OCEAN requires no exhaustive partition extraction, but its per-query optimization becomes increasingly costly as the ensemble grows. LIRE remains computationally efficient by restricting its search to live regions, at the cost of counterfactual quality. Indeed, as detailed in Tables~\ref{tab:ft_lire_vs_cfmaps_nb_trees} and~\ref{tab:ft_lire_vs_cfmaps_depth}, LIRE returns increasingly suboptimal solutions, while FT also increasingly fails to find valid counterfactuals. Overall, these experiments show that CF-Maps maintains low online query costs as ensemble complexity increases, but also expose partition extraction as its main scalability bottleneck. We investigate this limitation more directly in the next section.

\subsection{Budgeted CF-Maps}\label{sec:results_CF_Maps_B}

Extracting a complete faithful partition of a tree ensemble (Algorithm~\ref{alg:build}) can be costly as the number of decision regions grows rapidly with ensemble size and depth. To address this bottleneck, we introduce a budgeted variant, \textbf{CF-Maps(B)}, which limits the number of regions extracted during preprocessing. The complete procedure is described in Appendix~\ref{app:budgetedCFMaps}.

CF-Maps(B) stops the region extraction once a prescribed budget of $N$ certified regions (i.e., hyperrectangles over which the prediction is certified constant) is reached, leaving part of the input space unresolved. The extracted certified regions are indexed as in CF-Maps, with one KD-tree per class, while unresolved regions are stored in an additional KD-tree. For each query, the nearest extracted certified region of the target class provides a valid counterfactual and an upper bound $U$ on the optimal distance. Similarly, the nearest unresolved region provides a distance $L$ such that $\min\{U,L\}$ is a lower bound on the optimal distance. When $U\leq L$, the returned counterfactual is certified globally optimal; otherwise, the difference $U-L$ bounds its possible suboptimality. CF-Maps(B) therefore replaces exhaustive preprocessing with a controlled computational budget while preserving validity and providing a rigorous certificate of counterfactual quality. 

\textbf{Result 5. CF-Maps(B) controls preprocessing cost as ensembles grow while retaining optimality guarantees or quality certificates.} We now evaluate CF-Maps and CF-Maps(B) on increasingly large and deep ensembles, together with LIRE and EECE, which are the best-performing heuristic baselines. We retain only the most difficult datasets (BC, PD and SE), since preprocessing is very fast on CP and FI, even for the original CF-Maps. Results across all datasets are reported in Table~\ref{tab:nps-method-comparison} in Appendix~\ref{app:nps-method-comparison} for completeness. We set the CF-Maps(B) budget to $N=10^6$ regions; Appendix~\ref{app:budget_ablation} studies the sensitivity to this parameter. Table~\ref{tab:scalability} reports preprocessing time, query time, and counterfactual distance as the number of trees and their maximum depth increase. 

\begin{table*}[!h]
\centering
\small
\setlength{\tabcolsep}{3pt}
\caption{Comparison of CF-Maps, CF-Maps(B) with \(N=10^6\), LIRE, and EECE. \(T_P\) is preprocessing time in seconds, \(T_Q\) is mean query time in milliseconds, and \(D\) is the distance ratio relative to the optimal distance. CF-Maps/CF-Maps(B) preprocessing includes faithful/budgeted partition extraction and KD-tree construction. 
The preprocessing definitions used for LIRE and EECE follow their respective implementations. 
Values are averaged over five seeds.}
\resizebox{\textwidth}{!}{%
\begin{tabular}{llccccccccccccc}
\toprule
Dataset & \multicolumn{2}{c}{Forest} & \multicolumn{3}{c}{CF-Maps} & \multicolumn{3}{c}{CF-Maps(B) ($N=10^6$)} & \multicolumn{3}{c}{LIRE} & \multicolumn{3}{c}{EECE} \\
\cmidrule(lr){4-6}\cmidrule(lr){7-9}\cmidrule(lr){10-12}\cmidrule(lr){13-15}
 &  Trees & Depth & $T_P$[s] & $T_Q$[ms] & $D$ & $T_P$[s] & $T_Q$[ms] & $D$ & $T_P$[s] & $T_Q$[ms] & $D$ & $T_P$[s] & $T_Q$[ms] & $D$ \\
\midrule
BC & 100 & 8 & 532.38 & 32.0 & 1.00 & 18.06 & 18.4 & 1.00 & 5.82 & 1.10 & 1.67 & 5.98 & 37.0 & 1.17 \\
\cline{2-15}
 & 100 & 10 & 595.31 & 37.5 & 1.00 & 18.04 & 17.4 & 1.00 & 5.82 & 1.10 & 1.68 & 5.98 & 38.9 & 1.16 \\
\cline{2-15}
 & 100 & 12 & 537.58 & 31.9 & 1.00 & 17.77 & 17.6 & 1.00 & 5.76 & 1.10 & 1.68 & 5.92 & 37.1 & 1.16 \\
\cline{2-15}
 & 200 & 10 & 1009.62 & 37.0 & 1.00 & 30.14 & 19.0 & 1.00 & 11.35 & 1.10 & 1.69 & 11.65 & 75.5 & 1.15 \\
\cline{2-15}
 & 500 & 10 & 2493.35 & 36.4 & 1.00 & 75.69 & 21.6 & 1.00 & 20.90 & 1.20 & 1.71 & 21.73 & 161.7 & 1.12 \\
\cline{2-15}
 & 1000 & 10 & 6633.01 & 34.4 & 1.00 & 180.09 & 23.1 & 1.00 & 37.47 & 1.20 & 1.75 & 39.64 & 384.1 & 1.11 \\
\midrule
PD & 100 & 8 & 459.01 & 13.3 & 1.00 & 24.32 & 7.30 & 1.00 & 14.03 & 1.60 & 2.48 & 14.44 & 56.6 & 1.10 \\
\cline{2-15}
 & 100 & 10 & 520.53 & 14.0 & 1.00 & 32.72 & 6.70 & 1.00 & 14.04 & 1.60 & 2.62 & 14.66 & 62.7 & 1.10 \\
\cline{2-15}
 & 100 & 12 & 604.85 & 15.2 & 1.00 & 35.40 & 6.50 & 1.00 & 14.04 & 1.60 & 2.78 & 14.74 & 64.2 & 1.10 \\
\cline{2-15}
 & 200 & 10 & 567.35 & 13.6 & 1.00 & 50.93 & 6.50 & 1.00 & 27.21 & 1.70 & 2.73 & 28.79 & 130.0 & 1.07 \\
\cline{2-15}
 & 500 & 10 & 970.36 & 13.7 & 1.00 & 133.59 & 6.80 & 1.00 & 52.19 & 1.80 & 2.80 & 61.98 & 466.1 & 1.06 \\
\cline{2-15}
 & 1000 & 10 & 2198.57 & 13.2 & 1.00 & 302.58 & 7.10 & 1.00 & 90.49 & 1.80 & 2.82 & 130.28 & 1494.2 & 1.05 \\
\midrule
SE & 100 & 8 & 102.79 & 7.90 & 1.00 & 2.44 & 3.10 & 1.00 & 3.38 & 0.50 & 2.35 & 3.48 & 38.3 & 1.39 \\
\cline{2-15}
 & 100 & 10 & 103.82 & 7.90 & 1.00 & 2.37 & 3.00 & 1.00 & 3.33 & 0.50 & 2.36 & 3.45 & 38.8 & 1.37 \\
\cline{2-15}
 & 100 & 12 & 107.31 & 8.20 & 1.00 & 2.36 & 2.90 & 1.00 & 3.36 & 0.50 & 2.36 & 3.48 & 45.1 & 1.37 \\
\cline{2-15}
 & 200 & 10 & 142.89 & 8.10 & 1.00 & 3.99 & 3.20 & 1.00 & 6.50 & 0.40 & 2.26 & 6.71 & 73.5 & 1.35 \\
\cline{2-15}
 & 500 & 10 & 197.78 & 8.60 & 1.00 & 9.53 & 3.60 & 1.00 & 12.65 & 0.50 & 2.34 & 13.27 & 177.7 & 1.26 \\
\cline{2-15}
 & 1000 & 10 & 322.11 & 9.20 & 1.00 & 21.70 & 3.60 & 1.00 & 22.74 & 0.50 & 2.38 & 24.22 & 393.8 & 1.21 \\
\bottomrule
\end{tabular}
}
\label{tab:scalability}
\end{table*}

The results highlight a clear scalability trade-off: offline preprocessing becomes increasingly costly as the ensemble size grows, whereas online query times remain on the order of milliseconds. 
On BC, for example, CF-Maps preprocessing time increases from $532$ seconds for 100 depth-8 trees to $6{,}633$ seconds for 1000 depth-10 trees. Yet, once the map is constructed, the average query time remains below $38$ milliseconds across all tested configurations. A similar behavior is observed on PD and SE. Thus, large ensembles primarily challenge the construction of the map rather than its subsequent use. CF-Maps(B) directly addresses this bottleneck. On BC with 1000 depth-10 trees, it reduces preprocessing time by $97.3\%$ while retaining an average query time below $24$ ms. Similar reductions are observed on PD and SE. Importantly, in all tested configurations, CF-Maps(B) still certifies $100\%$ of the returned counterfactuals as globally optimal: the lower bound $\min\{U,L\}$ (not displayed in the table) matches the upper bound $U$ in all settings, indicating a zero optimality gap. These results show that limiting the number of regions extracted can substantially reduce the preprocessing cost with CF-Maps(B) without sacrificing optimality in the regimes considered here.

\section{Related Work}
\label{sec:related_works}

We review prior work on counterfactual explanations for tree ensembles, distinguishing heuristic and exact methods, before discussing amortized inference. For broader surveys, we refer the reader to \citet{karimi2022survey}, \citet{guidotti2024counterfactual}, and \citet{verma2024counterfactual}.

Early counterfactual methods were largely model-agnostic~\citep{DBLP:journals/corr/abs-1711-00399}, whereas \emph{white-box} approaches tailored to specific hypothesis classes can exploit model structure to produce higher-quality explanations. For tree ensembles, many methods are \emph{heuristic}: they provide no guarantees on optimality and may even fail to find a feasible counterfactual. A widely used tree-specific baseline is Feature Tweaking (FT)~\citep{Tolomei_2017}, which explores alternative root-to-leaf paths in individual trees to propose feature changes that may flip the ensemble prediction. Since it does not reason jointly over all trees, FT can be suboptimal or invalid.

Other heuristics use differentiable approximations of tree splits~\citep{lucic2022focus}, restrict the search to \emph{live regions} supported by training data~\citep{carreira2023LIRE}, or approximate the ensemble with a single surrogate tree~\citep{fernandez2020random}. More recently, EECE~\citep{zhang2025efficient} improves FT by prioritizing candidate modifications using the ensemble structure and combining them with live regions to ensure feasibility.

These methods are often computationally efficient, but their lack of guarantees is not merely a technical issue. In recourse settings, heuristic counterfactuals may overstate the effort required, miss feasible alternatives, or yield uneven explanation quality across individuals, thereby introducing or amplifying unfairness. This motivates exact methods that provide globally optimal and comparable recourse costs. Several works pursue this goal for tree ensembles using mathematical programming or SAT/SMT formulations~\citep{cui2015,ijcai2020p395,karimi2020modelagnostic,parmentier2021optimal}. OCEAN~\citep{parmentier2021optimal} is the current state of the art in this category, relying on tight linear relaxations and flow-based encodings to compute optimal counterfactuals within seconds for moderately sized forests. Beyond mathematical programming, \citet{blanchart2021exact} builds a query-specific search tree over candidate regions obtained by intersecting split conditions, and prunes it by branch-and-bound. Despite their guarantees, these exact methods solve a new optimization problem for each query, limiting their use in interactive settings.

Finally, a smaller body of work studies amortized counterfactual inference, where upfront preprocessing enables fast query-time generation. These approaches either rely on differentiable models to exploit gradient information~\citep{mahajan2019preserving,van2021conditional,yang2021model,guo2021counternet}, or are model-agnostic and thus cannot exploit the internal structure of a given tree ensemble. The latter category includes GAN-based methods~\citep{nemirovsky2022countergan} as well as formulations based on Markov decision processes~\citep{verma2022amortized} or reinforcement learning~\citep{samoilescu2021model}. In a different vein, using only black-box access, \citet{becker2021step} construct a hierarchical partition of the input space by iteratively sampling within regions (starting from the full space), subdividing regions with inconsistent predictions, and then leveraging this hierarchy to accelerate counterfactual search at query time. Finally, \citet{rawal2020beyond} precompute a covering set of counterfactual examples that provides broad query coverage, but may be suboptimal for a given instance.

To the best of our knowledge, our method is the first to provide optimal counterfactual explanations for tree ensembles with amortized inference, by leveraging a counterfactual map representation that supports interactive use.

\section{Conclusions}

We have introduced the concept of \emph{counterfactual maps}, which provide a direct, query-conditioned view of recourse. Beyond pointwise explanations, counterfactual maps constitute a form of global explanation: they make explicit the geometric organization of decision regions in the input space and enable interactive exploration of how recourse varies across queries.

Using class-specific volumetric KD-trees, CF-Maps identifies the nearest target region and obtains an optimal counterfactual by projection. Across datasets with numerical, binary, and categorical features, and under different objective functions, CF-Maps achieves sublinear average query time, typically in the millisecond range, while preserving optimality. In contrast,
heuristic baselines (FT, LIRE, EECE, and DiCE) substantially overestimate minimal recourse costs (by factors up to $4.8$) or frequently fail to find feasible counterfactuals when they exist (up to a 93\% failure rate for FT). OCEAN also preserves optimality but requires orders of magnitude more time to generate 1000 explanations. To scale and address the preprocessing bottleneck, we introduce CF-Maps(B), a budgeted variant that reduces partition construction time while providing valid counterfactuals with certified optimality gaps. In our experiments, it reduces preprocessing time by up to 97.3\%, with all evaluated queries certified optimal.  Taken together, these results establish this methodology as a very promising solution for delivering trustworthy recourse in interactive, real-world deployments.

The research avenues stemming from this work are numerous.
On the algorithmic side, refinements could further accelerate partition extraction, which remains a preprocessing bottleneck despite the improvements brought by CF-Maps(B). Notably, tighter approximations of unresolved regions could reduce preprocessing while preserving provable guarantees. Another direction is to design gating mechanisms that handle most queries interactively while triggering exact cold-start optimization only in difficult cases. Beyond single-point recourse, the geometric structure of  counterfactual maps could be used to characterize Rashomon sets of near-optimal explanations, offering multiple alternatives with controlled recourse costs. Finally, extending counterfactual maps to other hypothesis classes is a compelling research avenue. The same methodology could apply to models whose decision functions are piecewise constant over axis-aligned regions (e.g., rule-based models, or models operating on discretized features). Extending the approach beyond this setting might also be possible, though it would likely require different geometric data structures.

\section*{Acknowledgment}
This research was enabled by support provided by Calcul Québec and the Digital Research Alliance of Canada, as well as funding from the SCALE AI Chair in Data-Driven Supply Chains. It was also supported by the Fonds de recherche du Québec–Nature et technologies (FRQNT) through a Team Research Project (327090). We thank Théo Guyard and Pierre-Yves Bouchet for their insightful discussions during the early stages of this work.
\bibliographystyle{plainnat}
\bibliography{references}

\appendix

\section{Proof of Proposition~\ref{prop:cf_maps_structure}}
\label{structure_proof}

\begin{proof}
Since $\mathbb{H}$ is finite, $\mathbb{H}_{y'}\subseteq\mathbb{H}$ is finite as well.
First, we highlight a closed form for $d_p(x,\mathcal{H})$.
Fix an axis-aligned hyperrectangle $\mathcal{H}=[a,b]\subseteq\mathbb{R}^m$ with $a\le b$. For $k=1,\ldots,m$, define
\[
\delta_k(x,\mathcal{H}) \triangleq \max\{0,\ a_k-x_k,\ x_k-b_k\},\qquad
\delta(x,\mathcal{H}) \triangleq (\delta_1(x,\mathcal{H}),\ldots,\delta_m(x,\mathcal{H})).
\]
Let $\Pi_{\mathcal{H}}(x)$ be the coordinate-wise clamp of $x$ to $[a,b]$, such that
$|x_k-(\Pi_{\mathcal{H}}(x))_k|=\delta_k(x,\mathcal{H})$ for all $k$. For any $p\in[1,\infty]$,
\begin{equation}
\label{eq:dp_delta_appendix}
d_p(x,\mathcal{H})=\inf_{x'\in\mathcal{H}}\|x-x'\|_p=\|x-\Pi_{\mathcal{H}}(x)\|_p=\|\delta(x,\mathcal{H})\|_p.
\end{equation}

\paragraph{Case $p\in\{1,\infty\}$.}
For each sign pattern $s\in\{-1,0,+1\}^m$, consider the polyhedron
\[
R_s(\mathcal{H}) \triangleq \bigl\{x:\ x_k\le a_k\ \text{if }s_k=-1;\ a_k\le x_k\le b_k\ \text{if }s_k=0;\ x_k\ge b_k\ \text{if }s_k=+1\bigr\}.
\]
On $R_s(\mathcal{H})$, each $\delta_k(\cdot,\mathcal{H})$ is affine. Hence, by~\eqref{eq:dp_delta_appendix},
\begin{itemize}
\item $d_1(\cdot,\mathcal{H})=\sum_k \delta_k(\cdot,\mathcal{H})$ is affine on $R_s(\mathcal{H})$;
\item $d_\infty(\cdot,\mathcal{H})=\max_k \delta_k(\cdot,\mathcal{H})$ is the maximum of finitely many affine functions on $R_s(\mathcal{H})$, and is therefore affine on a finite polyhedral refinement obtained by fixing an index attaining the maximum.
\end{itemize}
Consequently, for each $\mathcal{H}\in\mathbb{H}_{y'}$ and $p\in\{1,\infty\}$, there exists a finite polyhedral partition $\mathcal{P}(\mathcal{H})$ of $\mathbb{R}^m$ such that $d_p(\cdot,\mathcal{H})$ is affine on every cell of $\mathcal{P}(\mathcal{H})$.
Let
$
\smash{\mathcal{P}\triangleq \bigwedge_{\mathcal{H}\in\mathbb{H}_{y'}} \mathcal{P}(\mathcal{H})}
$
be the common refinement, i.e., the partition whose cells are all nonempty intersections of one cell from each partition $\mathcal{P}(\mathcal{H})$ with $\mathcal{H}\in\mathbb{H}_{y'}$. $\mathcal{P}$ is still a finite polyhedral partition, and for every $C\in\mathcal{P}$ and every $\mathcal{H}\in\mathbb{H}_{y'}$, the function $d_p(\cdot,\mathcal{H})$ is affine on $C$.

Fix $\mathcal{H}_i,\mathcal{H}_j\in\mathbb{H}_{y'}$. On each $C\in\mathcal{P}$, the bisector condition
$d_p(x,\mathcal{H}_i)=d_p(x,\mathcal{H}_j)$
is an equality of two affine functions, hence it defines either $C\cap H$ for some hyperplane $H$, or all of $C$ (when the affine functions coincide on $C$). Therefore
\[
B_{ij}^{(p)}=\{x:\ d_p(x,\mathcal{H}_i)=d_p(x,\mathcal{H}_j)\}
\]
is a finite union of polyhedra.

Similarly, for fixed $\mathcal{H}_i\in\mathbb{H}_{y'}$ and any $C\in\mathcal{P}$,
\[
\mathcal{V}_i^{(p)}\cap C
=
\{x\in C:\ d_p(x,\mathcal{H}_i)\le d_p(x,\mathcal{H}_j)\ \ \forall\,\mathcal{H}_j\in\mathbb{H}_{y'}\}
\]
is given by finitely many linear inequalities on $C$, hence is a polyhedron (possibly empty). Since $\mathcal{P}$ is finite, $\mathcal{V}_i^{(p)}$ is a finite union of polyhedra.

\paragraph{Case $p=2$.}
By~\eqref{eq:dp_delta_appendix},
\[
d_2(x,\mathcal{H})^2=\sum_{k=1}^m \delta_k(x,\mathcal{H})^2.
\]
On each $R_s(\mathcal{H})$, $\delta_k(\cdot,\mathcal{H})$ is affine, so $d_2(\cdot,\mathcal{H})^2$ is a quadratic polynomial there. Fix $\mathcal{H}_i,\mathcal{H}_j\in\mathbb{H}_{y'}$ and refine the sign partitions for $\mathcal{H}_i$ and $\mathcal{H}_j$; on each resulting polyhedral cell, the bisector condition
$d_2(x,\mathcal{H}_i)=d_2(x,\mathcal{H}_j)$
is equivalent to a single quadratic equation, hence $\smash{B_{ij}^{(2)}}$ is contained in a finite union of quadratic hypersurfaces.

Finally, Voronoi regions are not necessarily polyhedral when $p=2$. In $\mathbb{R}^2$, let $\mathbb{H}_{y'}=\{\mathcal{H}_1,\mathcal{H}_2\}$ with
$\mathcal{H}_1=\{0\}\times[0,1]$ and $\mathcal{H}_2=\{(1,0)\}$ (a degenerate rectangle).
On the strip $\{(x_1,x_2): x_1\ge 0,\ 0\le x_2\le 1\}$, one has
$d_2((x_1,x_2),\mathcal{H}_1)=x_1$ and
$d_2((x_1,x_2),\mathcal{H}_2)=\sqrt{(x_1-1)^2+x_2^2}$,
so the bisector satisfies $x_2^2=2x_1-1$, a parabola segment.
\end{proof}

\newpage 

\section{Illustrations of Counterfactual Maps under Different Norms}
\label{app:illustrative_example_extended}

\begin{figure*}[h!]
    \centering

    \begin{subfigure}[t]{0.495\textwidth}
    \centering\includegraphics[width=0.98\textwidth]{plots/decisionBoundaryRects.pdf}\vspace*{-0.35cm}
    \caption{Decision boundaries of the model.}
   \label{fig:decision_boundary_bis}
  \end{subfigure}
    \hfill
    \begin{subfigure}[t]{0.495\textwidth}
    \centering\includegraphics[width=0.98\textwidth]{plots/VD_norm_1.pdf}\vspace*{-0.35cm}
  \caption{Counterfactual map for the $L_1$ norm.}
   \label{fig:cf_map_l1}
  \end{subfigure}

      \begin{subfigure}[t]{0.495\textwidth}
    \centering\includegraphics[width=0.98\textwidth]{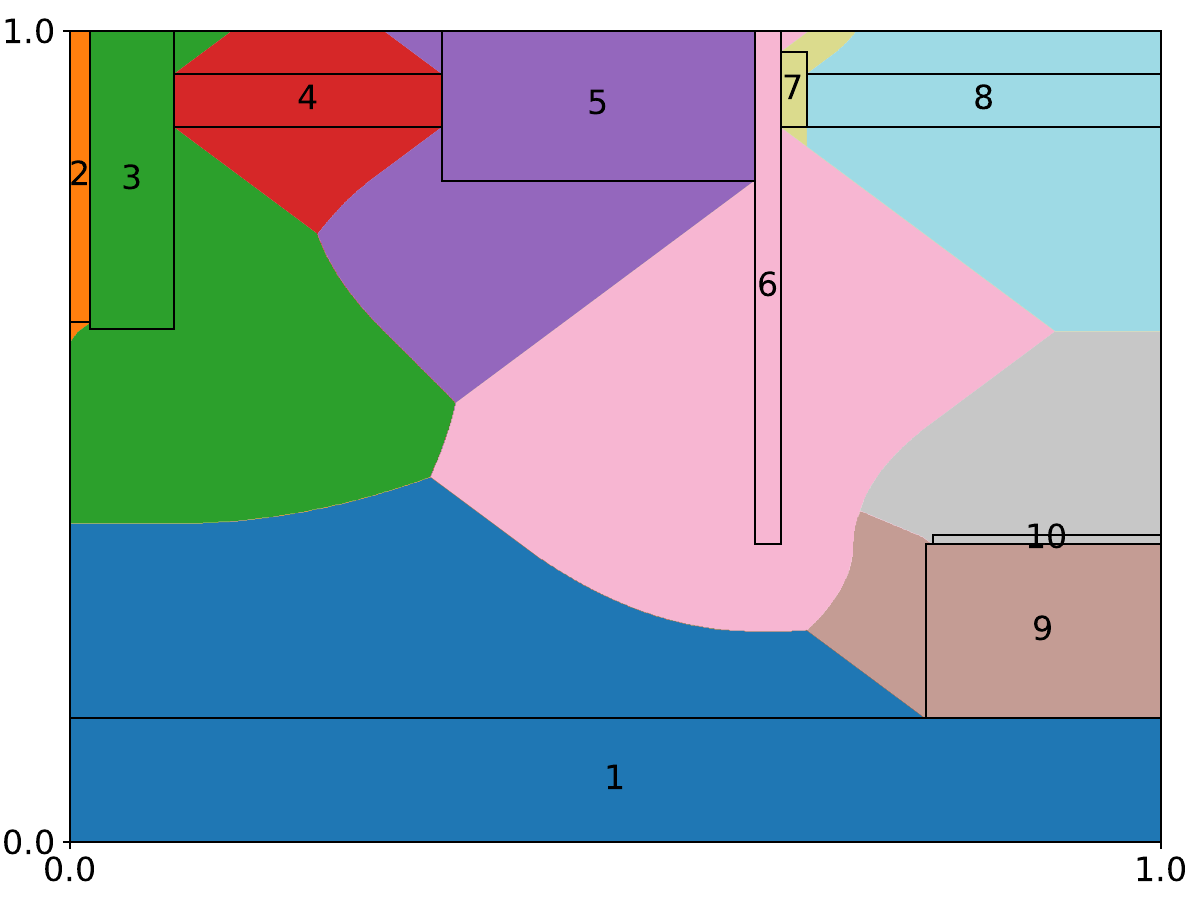}\vspace*{-0.35cm}
  \caption{Counterfactual map for the $L_2$ norm.}
   \label{fig:cf_map_l2}
  \end{subfigure}
      \begin{subfigure}[t]{0.495\textwidth}
    \centering\includegraphics[width=0.98\textwidth]{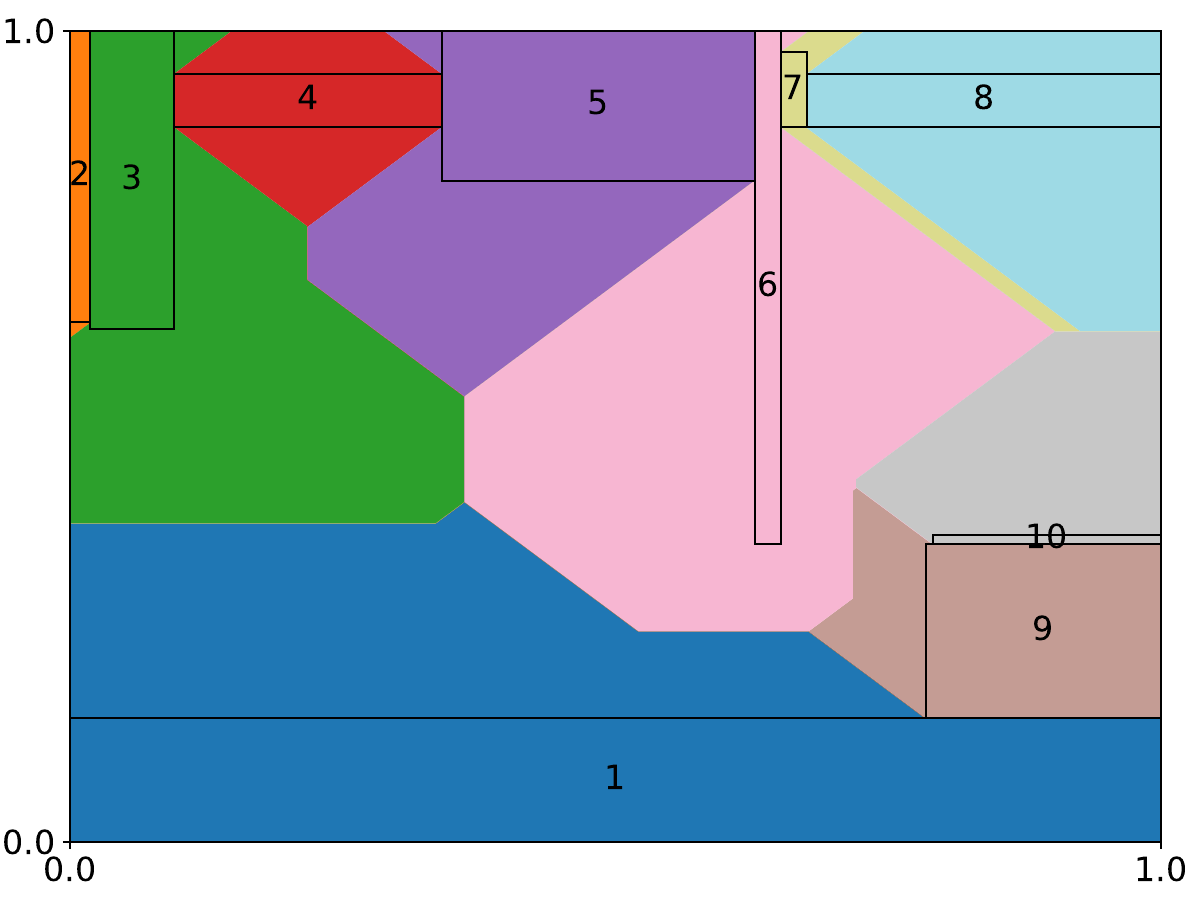}\vspace*{-0.35cm}
  \caption{Counterfactual map for the $L_\infty$ norm.}
   \label{fig:cf_map_linf}
  \end{subfigure}

    \caption{Counterfactual maps using different norms for target class \smash{\colorbox[HTML]{7F7F7F}{\textcolor{white}{$y_2$}}}.
    For any input query and distance, the map identifies the closest hyperrectangle of this class; projecting the query onto that region immediately yields a globally optimal counterfactual. This simple two-dimensional example uses a random forest with two depth-5 trees trained on the ``blobs'' dataset.}
    \label{fig:illustrative_example_extended}
\end{figure*}

\section{Proof of Theorem~\ref{thm:kd_exact}} 
\label{kd_exact_proof}

\subsection{Setting}

Fix an integer $m \ge 1$ and a norm index $1 \le p \le \infty$. We consider the metric space
\[
  (\mathbb{R}^m, d_p), \qquad d_p(x,x') \triangleq \|x - x'\|_p.
\]

Let $\mathbb{H}_{y'} = \{\mathcal{H}_1, \dots, \mathcal{H}_n\}$ be the axis-aligned hyperrectangles
associated with target class $y'$ in $\mathbb{R}^m$, where each
\[
  \mathcal{H}_i \triangleq [a^{(i)}, b^{(i)}]
  = \bigl\{ z \in \mathbb{R}^m : a^{(i)}_k \le z_k \le b^{(i)}_k \text{ for all } k = 1, \dots, m \bigr\},
\]
with $a^{(i)}_k \le b^{(i)}_k$ for all $k \in \{1, \ldots, m\}$.

For any nonempty set $S \subseteq \mathbb{R}^m$ and any $x \in \mathbb{R}^m$, define
\[
  d_p(x, S) \triangleq \inf_{x' \in S} \|x - x'\|_p.
\]

Given a query point $x \in \mathbb{R}^m$, we want to find an index
\[
  i^\star \in \{1,\dots,n\}
\]
such that
\[
  d_p(x, \mathcal{H}_{i^\star}) = \min_{1 \le i \le n} d_p(x, \mathcal{H}_i).
\]

\subsection{KD-Tree Structure}

We assume a KD-tree construction over $\mathbb{H}_{y'}$ producing a rooted binary tree $T$ whose nodes are denoted by $v$. For each node $v$, we associate:

\begin{itemize}
  \item a nonempty subset of rectangles $\mathbb{H}(v) \subseteq \mathbb{H}_{y'}$ (those stored in the subtree rooted at $v$);
  \item an axis-aligned bounding box
  \[
    B(v) \triangleq [\ell(v), u(v)] \subseteq \mathbb{R}^m
  \]
  such that
  \[
    \bigcup_{\mathcal{H} \in \mathbb{H}(v)} \mathcal{H} \subseteq B(v).
  \]
  note that $B(v)$ is defined as the smallest bounding box. 
\end{itemize}

Leaves $v$ satisfy $|\mathbb{H}(v)| \le k$ (for some fixed $k$), and we store those rectangles explicitly in the leaf. Internal nodes are built by splitting $\mathbb{H}(v)$ into two disjoint subsets $\mathbb{H}(v_L)$ and $\mathbb{H}(v_R)$ according to some coordinate and threshold (the KD-tree splitting rule).

The construction guarantees the following invariant.

\medskip

\noindent\textbf{Tree Bounding Box Invariant.} 
For every node $v$ and each child $c$ of $v$,
\[
  \mathbb{H}(c) \subseteq \mathbb{H}(v), 
  \qquad
  \bigcup_{\mathcal{H} \in \mathbb{H}(c)} \mathcal{H} \subseteq B(c) \subseteq B(v).
\]

In particular, each child bounding box is contained in the parent bounding box and contains all rectangles in that subtree.

\subsection{Query Algorithm}

Fix $x \in \mathbb{R}^m$ and define, for each node $v$ of the KD-tree,
\[
  lb(v) := d_p(x, B(v)).
\]
We call $lb(v)$ the \emph{node lower bound}.

The query algorithm \texttt{QueryNearest}$(x, \text{root}, p)$ is as follows.

\begin{itemize}
  \item Initialize:
  \[
    d^\star \gets +\infty, \qquad i^\star \gets \text{null}.
  \]
  Let $Q$ be a min-priority queue keyed by $lb(v)$. Insert the root node $r$ with key $lb(r)$.

  \item While $Q$ is not empty:
  \begin{enumerate}
    \item Extract a node $v$ with smallest key $lb(v)$ from $Q$.
    \item If $lb(v) \ge d^\star$, \emph{stop the loop}.
    \item If $v$ is a leaf:
      \begin{itemize}
        \item For each rectangle $\mathcal{H} \in \mathbb{H}(v)$,
        compute $d = d_p(x, \mathcal{H})$. If $d < d^\star$, then set $d^\star \gets d$ and $i^\star$ to be the index of $\mathcal{H}$.
      \end{itemize}
    \item If $v$ is an internal node:
      \begin{itemize}
        \item For each child $c$ of $v$, compute $lb(c) = d_p(x, B(c))$. If $lb(c) < d^\star$, insert $c$ in $Q$ with key $lb(c)$.
      \end{itemize}
  \end{enumerate}
\end{itemize}

At the end, the algorithm returns $(i^\star, d^\star)$.

\subsection{Lemmas}

\begin{lemma}[Distance monotonicity under inclusion]\label{lem:dis-monot}
Let $S, T \subseteq \mathbb{R}^m$ with $S \subseteq T \neq \varnothing$. Then, for all $x \in \mathbb{R}^m$,
\[
  d_p(x, T) \le d_p(x, S).
\]
\end{lemma}

\begin{proof}
By definition,
\[
  d_p(x, T) = \inf_{x' \in T} \|x - x'\|_p.
\]
Since $S \subseteq T$, we have
\[
  \inf_{x' \in T} \|x - x'\|_p 
  \le \inf_{x' \in S} \|x - x'\|_p,
\]
because the infimum is taken over a larger set. Thus $d_p(x, T) \le d_p(x, S)$.
\end{proof}

\begin{lemma}[Bounding box as a lower bound]\label{lem:bbox}
For any node $v$ and any rectangle $\mathcal{H} \in \mathbb{H}(v)$, we have
\[
  lb(v) = d_p(x, B(v)) \le d_p(x, \mathcal{H}).
\]
\end{lemma}

\begin{proof}
By the tree bounding box invariant, $\mathcal{H} \subseteq B(v)$. Applying Lemma~\ref{lem:dis-monot} with $S = \mathcal{H}$ and $T = B(v)$ yields
\[
  d_p(x, B(v)) \le d_p(x, \mathcal{H}),
\]
i.e., $lb(v) \le d_p(x, \mathcal{H})$.
\end{proof}

\begin{lemma}[Best-so-far distance over visited rectangles]\label{lem:best_so_far}
At any time during the execution of the algorithm,
\[
  d^\star 
  = \min \{ d_p(x, \mathcal{H}_i) : \mathcal{H}_i \text{ has been explicitly evaluated so far} \},
\]
with the convention $\min \emptyset = +\infty$ at initialization.
\end{lemma}

\begin{proof}
Initially, no rectangle has been evaluated and $d^\star = +\infty$, so the statement holds.

Whenever the algorithm evaluates a new rectangle $\mathcal{H} \in \mathbb{H}_{y'}$ (this happens only in leaves), it computes its distance $d = d_p(x, \mathcal{H})$. If $d < d^\star$, it updates $d^\star \gets d$; otherwise $d^\star$ remains unchanged. This is exactly the standard maintenance of a running minimum. Therefore, after each update, $d^\star$ is the minimum over all rectangles evaluated so far.
\end{proof}

We now formalize a coverage property for unvisited rectangles.

\begin{lemma}[Coverage invariant]
\label{lem:coverage}
Call a rectangle \emph{visited} if its distance has been computed (i.e., it has been iterated over in some leaf). At any iteration of the main loop, every unvisited rectangle $\mathcal{H}$ satisfies at least one of the following:
\begin{itemize}
  \item $\mathcal{H} \in \mathbb{H}(v)$ for some node $v$ currently in the priority queue $Q$; or
  \item there exists a node $u$ such that $\mathcal{H} \in \mathbb{H}(u)$ and $lb(u) \ge d^\star$, i.e., $\mathcal{H}$ cannot yield a distance strictly smaller than $d^\star$.
\end{itemize}
\end{lemma}

\begin{proof}
We proceed by induction on the number of iterations of the main loop.

\emph{Initialization.}
At the beginning, no rectangle is visited, and $Q$ contains only the root $r$. Since $\mathbb{H}(r) = \mathbb{H}_{y'}$, every rectangle belongs to $\mathbb{H}(r)$, and the first condition holds.

\emph{Inductive step.}
Assume the lemma holds at the beginning of an iteration. Let $v$ be the node extracted from $Q$ in that iteration. We consider cases:

\begin{itemize}
  \item If the algorithm stops immediately because $lb(v) \ge d^\star$, then no further iteration occurs, so we are done.
  \item If $v$ is a leaf: then all rectangles in $\mathbb{H}(v)$ are evaluated in this iteration and thus become visited. Rectangles not in $\mathbb{H}(v)$ are unaffected and remain covered as in the inductive hypothesis. The lemma continues to hold.
  \item If $v$ is internal: then $\mathbb{H}(v)$ is partitioned into the children subsets, e.g.,
  \[
    \mathbb{H}(v) = \mathbb{H}(v_L) \cup \mathbb{H}(v_R),
    \qquad
    \mathbb{H}(v_L) \cap \mathbb{H}(v_R) = \varnothing.
  \]
  For each child $c$, we compute $lb(c)$ and insert $c$ into $Q$ only if $lb(c) < d^\star$. Thus any rectangle $\mathcal{H} \in \mathbb{H}(v)$ for which its child $c$ (i.e., child $c$ such that $\mathcal{H} \in \mathbb{H}(c)$) is inserted into $Q$ is covered by a node in $Q$ (the first condition). 

  If, on the contrary, a child $c$ is not inserted because $lb(c) \ge d^\star$, then for any $\mathcal{H} \in \mathbb{H}(c)$, Lemma~\ref{lem:bbox} gives
  \[
    d_p(x, \mathcal{H}) \ge lb(c) \ge d^\star,
  \]
  so such rectangles satisfy the second condition in the statement of this lemma.
\end{itemize}

Thus, after handling node $v$, each unvisited rectangle still satisfies one of the two conditions, and the lemma holds for the next iteration (or the algorithm terminates).
\end{proof}

\subsection{Main Optimality Theorem}

\begin{theorem}[Exactness of KD-Tree Nearest-Hyperrectangle Search]
Fix $x \in \mathbb{R}^m$ and $1 \le p \le \infty$. Run \texttt{QueryNearest}$(x, \text{root}, p)$ on the KD-tree built over $\mathbb{H}_{y'} = \{\mathcal{H}_1, \dots, \mathcal{H}_n\}$ as described above, and let the algorithm terminate with output $(i^\star, d^\star)$. Then
\[
  d^\star = d_p(x, \mathcal{H}_{i^\star}) = \min_{1 \le i \le n} d_p(x, \mathcal{H}_i).
\]
In other words, the algorithm returns a rectangle whose distance to $x$ is globally minimal among rectangles in $\mathbb{H}_{y'}$.
\end{theorem}

\begin{proof}
By Lemma~\ref{lem:best_so_far}, at termination we have
\[
  d^\star 
  = \min \{ d_p(x, \mathcal{H}_i) : \mathcal{H}_i \text{ has been explicitly evaluated} \}.
\]
Thus it suffices to show:
\begin{equation}
  \label{eq:no-better-unvisited}
  \text{No unevaluated rectangle } \mathcal{H} \text{ satisfies } d_p(x, \mathcal{H}) < d^\star.
\end{equation}

Suppose, for contradiction, that there exists an unevaluated rectangle $\mathcal{H}$ with $d_p(x, \mathcal{H}) < d^\star$. Consider the iteration in which the algorithm terminates. There are two possible reasons for termination.

\medskip
\noindent\emph{Case 1: The queue $Q$ becomes empty.}

If $Q$ is empty, there is no node left to process. By Lemma~\ref{lem:coverage}, any unvisited rectangle $\mathcal{H}$ must then fall into the second condition, i.e., there exists a node $u$ such that $\mathcal{H} \in \mathbb{H}(u)$ and $lb(u) \ge d^\star$. By Lemma~\ref{lem:bbox},
\[
  d_p(x, \mathcal{H}) \ge lb(u) \ge d^\star,
\]
contradicting $d_p(x, \mathcal{H}) < d^\star$.

\medskip
\noindent\emph{Case 2: Early stopping when $lb(v) \ge d^\star$.}

The other termination condition is that the algorithm extracts a node $v$ from $Q$ with minimal key and finds that
\[
  lb(v) \ge d^\star,
\]
and thus breaks out of the loop.

Since $v$ has the smallest key in $Q$, every remaining node $w \in Q$ satisfies
\[
  lb(w) \ge lb(v) \ge d^\star.
\]
Let $\mathcal{H}$ be any unvisited rectangle. By Lemma~\ref{lem:coverage}, there exists either a node $w \in Q$ with $\mathcal{H} \in \mathbb{H}(w)$, or a node $u$ with $\mathcal{H} \in \mathbb{H}(u)$ and $lb(u) \ge d^\star$. In either case, applying Lemma~\ref{lem:bbox} yields
\[
  d_p(x, \mathcal{H}) \ge d^\star.
\]
Thus no unvisited rectangle can have distance strictly less than $d^\star$, contradicting the hypothesis and proving~\eqref{eq:no-better-unvisited}.

\medskip

Finally, in all possible termination scenarios, every rectangle $\mathcal{H}$ satisfies
\[
  d_p(x, \mathcal{H}) \ge d^\star,
\]
and at least one visited rectangle (namely $\mathcal{H}_{i^\star}$) satisfies $d_p(x, \mathcal{H}_{i^\star}) = d^\star$. Therefore,
\[
  d^\star = \min_{1 \le i \le n} d_p(x, \mathcal{H}_i),
\]
and the algorithm is exact.
\end{proof}

\FloatBarrier

\section{Additional Experimental Results} \label{sec:additional_results}

\subsection{Detailed Datasets Table}\label{app:datasets_detailed}

\begin{table}[h!]
\centering
\small
\caption{Summary of the datasets used in our experiments. Columns \(N\), \(O\), \(C\), and \(B\) report the number of numerical, ordinal, categorical, and binary features, respectively. Categorical features are one-hot encoded during preprocessing, resulting in a total of \#$\mathcal{X}$ (OHE) features and \#$\mathcal{Y}$ classes.}
\label{tab:full_datasets}
\begin{tabular}{lrr|rrrr|r|r}
\toprule
Dataset & \# samples & \# features & \(N\) & \(O\) & \(C\) & \(B\) & \#$\mathcal{X}$ (OHE) & \#$\mathcal{Y}$ \\
\midrule
BC: Breast Cancer Wisconsin &    683 &  9 &  0 &  9 & 0 &  0 &  9 & 2 \\
CP: COMPAS                  &  6,907 &  6 &  0 &  0 & 2 &  4 &  13 & 2 \\
FI: FICO                    & 10,459 & 10 & 0 &  0 & 6 &  4 & 23 & 2 \\
PD: Pima-Diabetes           &    768 & 8 & 8 & 0 & 0 &  0 & 8 & 2 \\
SE: Seeds                   &    210 & 7 & 7 & 0 & 0 &  0 & 7 & 3 \\

\bottomrule
\end{tabular}
\end{table}

\subsection{Examples of Counterfactuals} \label{sec:examples_cfs}

\def\modified#1{\textcolor{blue}{\textbf{#1}}}
\begin{table}[h!]
\centering
\caption{Qualitative examples of counterfactuals for the BC dataset, where features modified with respect to the query are displayed in \modified{blue}.}
\resizebox{\textwidth}{!}{%
\begin{tabular}{lrrrrrrrrr|r}
\toprule
Row & Clump-T & Uniformity-Size & Uniformity-Shape & Adhesion & Cell-Size & Bare-Nuclei & Bland-Chromatin & Normal-Nucleoli & Mitoses & cost \\
\midrule
\multicolumn{11}{c}{Query \#1 (0 $\rightarrow$ 1)} \\
Query & 5 & 1 & 2 & 1 & 2 & 1 & 3 & 1 & 1 & 0 \\
\midrule
OCEAN & 5 & \modified{3} & \modified{3} & 1 & \modified{3} & \modified{3} & 3 & \modified{3} & 1 & 8 \\
CF-Maps & 5 & \modified{3} & \modified{5} & 1 & \modified{3} & \modified{3} & 3 & 1 & 1 & 8 \\
FT [invalid] & 5 & 1 & 2 & \modified{6} & 2 & 1 & 3 & 1 & 1 & 5 \\
LIRE & 5 & \modified{3} & \modified{3} & 1 & \modified{3} & \modified{3} & 3 & \modified{3} & \modified{3} & 10 \\
DiCE & 5 & 1 & \modified{6} & 1 & 2 & \modified{5} & 3 & \modified{3} & \modified{3} & 12 \\
EECE & 5 & 1 & \modified{5} & 1 & 2 & \modified{7} & 3 & 1 & 1 & 9 \\
\midrule
\midrule
\multicolumn{11}{c}{Query \#2 (0 $\rightarrow$ 1)} \\
Query & 5 & 3 & 3 & 4 & 2 & 4 & 3 & 4 & 1 & 0 \\
\midrule
OCEAN & 5 & 3 & 3 & 4 & 2 & \modified{5} & 3 & 4 & 1 & 1 \\
CF-Maps & 5 & 3 & 3 & 4 & 2 & \modified{5} & 3 & 4 & 1 & 1 \\
FT & 5 & 3 & 3 & \modified{6} & 2 & 4 & 3 & 4 & 1 & 2 \\
LIRE & 5 & 3 & \modified{4} & 4 & \modified{3} & \modified{5} & 3 & \modified{7} & 1 & 6 \\
DiCE & 5 & 3 & 3 & 4 & 2 & \modified{7} & 3 & \modified{5} & 1 & 4 \\
EECE & 5 & 3 & 3 & 4 & 2 & \modified{5} & 3 & 4 & 1 & 1 \\
\midrule
\midrule
\multicolumn{11}{c}{Query \#3 (1 $\rightarrow$ 0)} \\
Query & 5 & 3 & 5 & 1 & 8 & 10 & 5 & 3 & 1 & 0 \\
\midrule
OCEAN & 5 & \modified{2} & \modified{4} & 1 & 8 & 10 & \modified{4} & \modified{2} & 1 & 4 \\
CF-Maps & 5 & \modified{2} & \modified{4} & 1 & 8 & 10 & \modified{4} & \modified{2} & 1 & 4 \\
FT [invalid] & 5 & 3 & \modified{4} & 1 & 8 & 10 & \modified{4} & 3 & 1 & 2 \\
LIRE & 5 & \modified{4} & 5 & 1 & 8 & \modified{2} & \modified{4} & \modified{6} & 1 & 13 \\
DiCE & 5 & \modified{4} & 5 & 1 & 8 & \modified{1} & \modified{3} & \modified{6} & 1 & 15 \\
EECE & 5 & \modified{2} & \modified{2} & 1 & 8 & 10 & \modified{4} & 3 & 1 & 5 \\
\midrule
\midrule
\multicolumn{11}{c}{Query \#4 (1 $\rightarrow$ 0)} \\
Query & 10 & 4 & 3 & 10 & 4 & 10 & 10 & 1 & 1 & 0 \\
\midrule
OCEAN & 10 & \modified{2} & \modified{2} & 10 & 4 & \modified{4} & 10 & 1 & 1 & 9 \\
CF-Maps & 10 & \modified{2} & \modified{2} & 10 & 4 & \modified{4} & 10 & 1 & 1 & 9 \\
FT [invalid] & 10 & 4 & \modified{2} & 10 & 4 & 10 & 10 & 1 & 1 & 1 \\
LIRE & \modified{6} & \modified{2} & \modified{2} & 10 & 4 & \modified{5} & \modified{2} & 1 & 1 & 20 \\
DiCE & \modified{5} & \modified{1} & \modified{2} & 10 & 4 & \modified{5} & \modified{2} & 1 & 1 & 22 \\
EECE & \modified{9} & \modified{2} & 3 & 10 & 4 & \modified{4} & 10 & 1 & 1 & 9 \\
\bottomrule
\end{tabular}%
}
\end{table}
\FloatBarrier

\subsection{KD-Trees built by CF-Maps}\label{sec:app_kdtree}

\begin{table}[h!]
    \centering
    \caption{Statistics of the computed KD-Trees, for random forests made of 100 depth-5 trees. For each dataset, we report the number of nodes in the KD-Tree, its memory footprint, as well as the actual number of visited nodes when answering counterfactual queries. We report both the average and standard deviation of each measure.}\label{tab:kd-tree}
    \begin{tabular}{lrrr}
    \toprule
    Dataset & Nodes & Memory [MB] & Visited/query \\
    \midrule
    BC & $(2.48 \pm 0.58) \times 10^{6}$ & $(1.80 \pm 0.42) \times 10^{3}$ & 2738.8 $\pm$ 1026.2 \\
    CP & $(2.80 \pm 0.18) \times 10^{2}$ & $(2.30 \pm 0.20) \times 10^{-2}$ & 17.5 $\pm$ 2.1 \\
    FI & $(4.79 \pm 0.14) \times 10^{3}$ & $(5.02 \pm 0.15) \times 10^{0}$ & 85.0 $\pm$ 9.3 \\
    PD & $(4.54 \pm 0.82) \times 10^{6}$ & $(3.21 \pm 0.58) \times 10^{3}$ & 711.3 $\pm$ 177.4 \\
    SE & $(1.64 \pm 0.34) \times 10^{6}$ & $(1.12 \pm 0.23) \times 10^{3}$ & 546.0 $\pm$ 91.7 \\
    \bottomrule
    \end{tabular}
   \end{table}

\subsection{Complementary Results}\label{app:complementary_results}
\begin{figure}[!h]
    \centering
    \includegraphics[width=\linewidth]{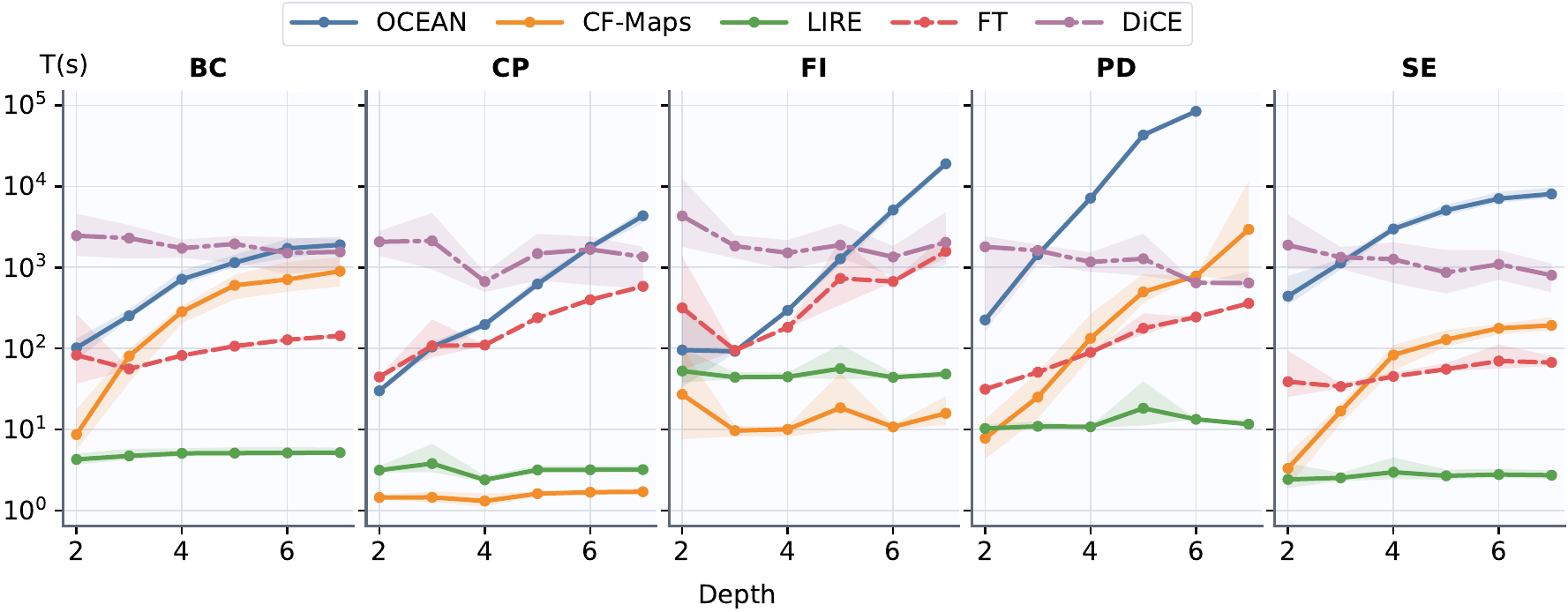}
\caption{Total time (including preprocessing) for generating 1000 counterfactuals, for random forests with 100 trees of varying depth. The shaded area shows the range between the fastest and slowest runs over the five random seeds.} \label{fig:time_vs_depth}
\end{figure}

\begin{table}[htbp]
\centering
\small
\setlength{\tabcolsep}{6pt}
\caption{Counterfactual quality of FT, LIRE, DiCE and EECE relative to CF-Maps, for varying numbers of depth-5 trees. The total distance ratio $D$ of each method relative to the optimal distance is computed over \emph{valid} counterfactuals only, whereas the failure rate (\%Fail) is computed over all explained points.} \label{tab:ft_lire_vs_cfmaps_nb_trees}
\begin{tabular}{llcrcrcrcr}
    \toprule
    \multirow{2}{*}{Dataset} & \multirow{2}{*}{\#Trees} & \multicolumn{2}{c}{FT} & \multicolumn{2}{c}{LIRE} & \multicolumn{2}{c}{DiCE} & \multicolumn{2}{c}{EECE} \\
    \cmidrule(lr){3-4}\cmidrule(lr){5-6}\cmidrule(lr){7-8}\cmidrule(lr){9-10}
     &  & $D$ & \%Fail & $D$ & \%Fail & $D$ & \%Fail & $D$ & \%Fail \\
    \midrule
    BC & 5 & 2.4 & 69.2 & 1.7 & 0.0 & 3.5 & 1.6 & 1.3 & 0.0 \\
     & 10 & 2.1 & 82.4 & 2.2 & 0.0 & 3.6 & 2.6 & 1.3 & 0.0 \\
     & 20 & 2.0 & 88.8 & 2.2 & 0.0 & 3.1 & 1.2 & 1.3 & 0.0 \\
     & 50 & 1.8 & 92.1 & 2.3 & 0.0 & 2.9 & 1.8 & 1.3 & 0.0 \\
     & 100 & 1.6 & 93.2 & 2.3 & 0.0 & 2.8 & 2.2 & 1.2 & 0.0 \\
    \midrule
    CP & 5 & 1.1 & 36.4 & 1.1 & 0.0 & 1.1 & 14.3 & 1.1 & 0.0 \\
     & 10 & 1.0 & 38.4 & 1.1 & 0.0 & 1.0 & 15.7 & 1.1 & 0.0 \\
     & 20 & 1.0 & 41.3 & 1.1 & 0.0 & 1.0 & 14.8 & 1.1 & 0.0 \\
     & 50 & 1.0 & 40.4 & 1.0 & 0.0 & 1.0 & 15.3 & 1.0 & 0.0 \\
     & 100 & 1.0 & 41.4 & 1.1 & 0.0 & 1.0 & 15.1 & 1.0 & 0.0 \\
    \midrule
    FI & 5 & 1.1 & 55.1 & 1.2 & 0.0 & 1.3 & 14.9 & 1.2 & 0.0 \\
     & 10 & 1.0 & 60.0 & 1.3 & 0.0 & 1.2 & 15.1 & 1.3 & 0.0 \\
     & 20 & 1.0 & 64.0 & 1.4 & 0.0 & 1.3 & 17.0 & 1.3 & 0.0 \\
     & 50 & 1.0 & 63.0 & 1.3 & 0.0 & 1.3 & 15.6 & 1.3 & 0.0 \\
     & 100 & 1.0 & 63.0 & 1.3 & 0.0 & 1.3 & 16.6 & 1.3 & 0.0 \\
    \midrule
    PD & 5 & 3.9 & 55.5 & 1.7 & 0.0 & 5.8 & 0.4 & 1.2 & 0.0 \\
     & 10 & 4.1 & 67.4 & 2.2 & 0.0 & 5.4 & 0.7 & 1.2 & 0.0 \\
     & 20 & 4.4 & 73.2 & 2.4 & 0.0 & 5.1 & 0.8 & 1.2 & 0.0 \\
     & 50 & 3.6 & 77.8 & 2.5 & 0.0 & 4.7 & 0.7 & 1.2 & 0.0 \\
     & 100 & 3.2 & 79.3 & 2.4 & 0.0 & 4.4 & 1.0 & 1.1 & 0.0 \\
    \midrule
    SE & 5 & 2.8 & 56.2 & 2.5 & 0.0 & 5.5 & 1.4 & 2.6 & 0.0 \\
     & 10 & 2.6 & 64.7 & 2.9 & 0.0 & 5.3 & 1.9 & 2.6 & 0.0 \\
     & 20 & 2.4 & 68.4 & 3.4 & 0.0 & 5.2 & 1.3 & 2.6 & 0.0 \\
     & 50 & 2.0 & 75.0 & 3.4 & 0.0 & 4.9 & 1.2 & 2.4 & 0.0 \\
     & 100 & 1.9 & 78.5 & 3.4 & 0.0 & 4.8 & 1.0 & 2.3 & 0.0 \\
    \bottomrule
    \end{tabular}
\end{table}

\begin{table}[htbp]
\centering
\small
\setlength{\tabcolsep}{6pt}
\caption{Counterfactual quality of FT, LIRE, DiCE and EECE relative to CF-Maps, for 100 trees of varying depths. The total distance ratio $D$ of each method relative to the optimal distance is computed over \emph{valid} counterfactuals only, whereas the failure rate (\%Fail) is computed over all explained points.} \label{tab:ft_lire_vs_cfmaps_depth}
\begin{tabular}{llcrcrcrcr}
    \toprule
    \multirow{2}{*}{Dataset} & \multirow{2}{*}{Depth} & \multicolumn{2}{c}{FT} & \multicolumn{2}{c}{LIRE} & \multicolumn{2}{c}{DiCE} & \multicolumn{2}{c}{EECE} \\
    \cmidrule(lr){3-4}\cmidrule(lr){5-6}\cmidrule(lr){7-8}\cmidrule(lr){9-10}
     &  & $D$ & \%Fail & $D$ & \%Fail & $D$ & \%Fail & $D$ & \%Fail \\
    \midrule
    BC & 2 & 1.3 & 94.0 & 2.1 & 0.0 & 3.1 & 2.3 & 1.6 & 0.0 \\
     & 3 & 1.8 & 94.4 & 2.2 & 0.0 & 2.9 & 2.6 & 1.5 & 0.0 \\
     & 4 & 1.7 & 93.7 & 2.2 & 0.0 & 2.8 & 1.9 & 1.3 & 0.0 \\
     & 5 & 1.6 & 93.2 & 2.3 & 0.0 & 2.8 & 2.2 & 1.2 & 0.0 \\
     & 6 & 1.6 & 93.4 & 2.3 & 0.0 & 2.8 & 1.7 & 1.2 & 0.0 \\
     & 7 & 1.5 & 93.1 & 2.2 & 0.0 & 2.8 & 2.0 & 1.2 & 0.0 \\
    \midrule
    CP & 2 & 1.0 & 42.9 & 1.1 & 0.0 & 1.1 & 15.9 & 1.1 & 0.0 \\
     & 3 & 1.0 & 47.6 & 1.0 & 0.0 & 1.1 & 14.5 & 1.1 & 0.0 \\
     & 4 & 1.0 & 42.4 & 1.1 & 0.0 & 1.0 & 13.9 & 1.0 & 0.0 \\
     & 5 & 1.0 & 41.4 & 1.1 & 0.0 & 1.0 & 15.1 & 1.0 & 0.0 \\
     & 6 & 1.0 & 42.7 & 1.0 & 0.0 & 1.0 & 12.8 & 1.0 & 0.0 \\
     & 7 & 1.0 & 42.1 & 1.0 & 0.0 & 1.0 & 14.5 & 1.0 & 0.0 \\
    \midrule
    FI & 2 & 1.0 & 60.3 & 1.4 & 0.0 & 1.3 & 20.4 & 1.4 & 0.0 \\
     & 3 & 1.0 & 64.2 & 1.4 & 0.0 & 1.3 & 18.7 & 1.4 & 0.0 \\
     & 4 & 1.0 & 56.9 & 1.3 & 0.0 & 1.3 & 17.8 & 1.3 & 0.0 \\
     & 5 & 1.0 & 63.0 & 1.3 & 0.0 & 1.3 & 16.6 & 1.3 & 0.0 \\
     & 6 & 1.0 & 63.4 & 1.3 & 0.0 & 1.3 & 15.8 & 1.3 & 0.0 \\
     & 7 & 1.0 & 63.5 & 1.3 & 0.0 & 1.2 & 12.3 & 1.3 & 0.0 \\
    \midrule
    PD & 2 & 1.6 & 85.4 & 1.7 & 0.0 & 3.4 & 2.0 & 1.3 & 0.0 \\
     & 3 & 2.4 & 83.0 & 2.0 & 0.0 & 3.8 & 1.8 & 1.2 & 0.0 \\
     & 4 & 2.9 & 80.0 & 2.2 & 0.0 & 4.1 & 1.2 & 1.1 & 0.0 \\
     & 5 & 3.2 & 79.3 & 2.4 & 0.0 & 4.4 & 1.0 & 1.1 & 0.0 \\
     & 6 & 4.3 & 81.8 & 2.3 & 0.0 & 4.0 & 0.6 & 1.1 & 0.0 \\
     & 7 & 3.9 & 77.6 & 2.8 & 0.0 & 5.1 & 0.6 & 1.1 & 0.0 \\
    \midrule
    SE & 2 & 1.6 & 81.9 & 2.7 & 0.0 & 4.1 & 1.4 & 2.4 & 0.0 \\
     & 3 & 1.8 & 80.5 & 3.0 & 0.0 & 4.3 & 1.5 & 2.3 & 0.0 \\
     & 4 & 1.9 & 79.4 & 3.2 & 0.0 & 4.5 & 1.4 & 2.3 & 0.0 \\
     & 5 & 1.9 & 78.5 & 3.4 & 0.0 & 4.8 & 1.0 & 2.3 & 0.0 \\
     & 6 & 1.9 & 78.6 & 3.4 & 0.0 & 4.8 & 1.3 & 2.2 & 0.0 \\
     & 7 & 2.0 & 78.3 & 3.5 & 0.0 & 4.9 & 1.0 & 2.2 & 0.0 \\
    \bottomrule
    \end{tabular}
\end{table}
\FloatBarrier

\section{Budgeted CF-Maps}\label{sec:cf_maps_b}

\subsection{Methodology}\label{app:budgetedCFMaps}

The exact CF-Maps methodology first extracts a complete faithful partition of the ensemble before constructing the geometric index. As shown in Section~\ref{sec:results}, this preprocessing can become expensive when the number of induced decision regions is large. CF-Maps(B) avoids requiring the complete partition: it progressively certifies homogeneous regions and stops once a prescribed budget $N$ of certified regions is reached. The remaining space is kept explicitly as unresolved regions. This section describes the construction and the resulting counterfactual guarantees.

\paragraph{Budgeted partition construction.}
Let $\mathcal T:\mathcal X\rightarrow\mathcal Y$ be the tree ensemble. We first collect the distinct split thresholds of each feature. These thresholds define elementary intervals whose Cartesian products form cells on which $\mathcal T$ is constant. Rather than enumerating this grid, CF-Maps(B) compactly represents a region $R$ by lower and upper interval indices along each feature and recursively refines these regions only when necessary.

The procedure maintains a priority queue of regions to explore. A region is \emph{live} if it contains at least one training point. Live regions are explored first, in decreasing order of the number of training points they contain. Ties, and regions without training points, are ordered by decreasing geometric score
\[
s(R)=\operatorname{diam}(R) +\max_{y'\in\mathcal Y}\delta_{y'}(R),
\]
where $\operatorname{diam}(R)$ denotes the diameter of the rectangle $R$ and 
\[
\delta_{y'}(R)=
\begin{cases}
d_2\left(\operatorname{center}(R),E_{y'}\right),
&\text{if a region of class $y'$ has already been certified},\\
1,&\text{otherwise}.
\end{cases}
\]
Here, $E_{y'}$ denotes the bounding box of the regions already certified as class $y'$. Moreover, we assume wlog that all the features have been normalized in $[0,1]$ using a min-max scaler (which is the case in all our experiments). Thus, exploration favors regions supported by the training data, while the geometric score encourages refinement of large regions and regions that extend beyond the parts of the decision space already covered. The envelopes $E_{y'}$ are used only to guide exploration and do not themselves represent certified decision regions.

Feature constraints are propagated during refinement. Binary features are restricted to $\{0,1\}$, and redundant thresholds inducing the same binary split are merged. For every one-hot group $G$, we enforce
\[
x_j\in\{0,1\}\quad (j\in G),
\qquad
\sum_{j\in G}x_j=1,
\]
and discard candidate splits producing infeasible branches.

\paragraph{Certifying a region.}
For each explored region $R$, CF-Maps(B) attempts to certify that the ensemble predicts a single class throughout its feasible portion. Agreement among sampled points is not sufficient for certification; instead, we use an analytical bound on the ensemble scores.

Let $\Lambda_t(R)$ denote the set of leaves of tree $t$ reachable from $R$, and let $p_{t,\lambda,y'}$ be the score assigned to class $y'$ at leaf $\lambda$ of tree $t$. For a candidate class $y'$ and any competing class $y\neq y'$, define
\[
m_{t,y',y}(R) =
\min_{\lambda\in\Lambda_t(R)}
\left(p_{t,\lambda,y'}-p_{t,\lambda,y}\right),
\qquad
M_{y',y}(R)=
\sum_{t=1}^{T}m_{t,y',y}(R).
\]
Note that for weighted tree ensembles, the weight associated to each tree $t$ can be directly incorporated into the calculation of $m_{t,y',y}(R)$. 

For every feasible $x\in R$,
\[
\sum_{t=1}^{T}
\left(
p_{t,\lambda_t(x),y'} -
p_{t,\lambda_t(x),y}
\right)
\ge M_{y',y}(R).
\]
Accounting for the ensemble's deterministic tie-breaking rule, $R$ can therefore be certified as class $y'$ whenever
\[
M_{y',y}(R)
\begin{cases}
\ge 0 \quad \text{and}& y'<y,\\
> 0 \quad \text{and}& y'>y,
\end{cases}
\qquad \forall y\neq y'.
\]
This condition is sufficient but not necessary: failure to certify a region only means that the bound cannot establish homogeneity, not that the region necessarily contains different predictions.

If certification fails, $R$ is split along a threshold of the original ensemble. We prioritize thresholds that occur most frequently among the reachable internal nodes, with ties favoring balanced children. The process continues until either the frontier is exhausted or $N$ regions have been certified. Importantly, $N$ bounds the number of \emph{certified} regions, rather than the number of explored or generated regions.

Algorithm~\ref{alg:certified-inner-approximation} summarizes the procedure.

\begin{algorithm}[H]
\small
\caption{Budgeted CF-Maps construction}
\label{alg:certified-inner-approximation}
\KwRequire{Tree ensemble $\mathcal T$, training inputs $X_{\mathrm{tr}}$, domain $\mathcal X$, certification budget $N$}
\KwEnsure{Certified regions $\{\mathcal C_{y'}\}_{y'\in\mathcal Y}$ and unresolved regions $\mathcal U$}

Construct the threshold grid and initialize the root region $R_0=\mathcal X$;

Associate training points and auxiliary samples with $R_0$; 

Initialize $\mathcal C_{y'}\gets\varnothing$ for all $y'\in\mathcal Y$, $\mathcal U\gets\varnothing$, and $n_{\mathrm{cert}}\gets0$;

Initialize priority queue $Q$ with $R_0$;

\While{$Q\neq\varnothing$ \textbf{and} $n_{\mathrm{cert}}<N$}{
Pop the highest-priority region $R$ from $Q$;

$y'\gets\mathrm{Certify}(\mathcal T,R)$;

\eIf{$y'\neq\bot$}{ 
    $\mathcal C_{y'}\gets\mathcal C_{y'}\cup\{R\}$\;
    $n_{\mathrm{cert}}\gets n_{\mathrm{cert}}+1$\;
    Update $E_{y'}$\;
}{
    $(j,\theta)\gets\mathrm{SelectFeasibleSplit}(\mathcal T,R)$\;
    \eIf{$(j,\theta)=\bot$}{
        $\mathcal U\gets\mathcal U\cup\{R\}$\;
    }{
        Split $R$ into feasible children $R_L,R_R$\;
        Propagate feature constraints and associated samples\;
        Insert $R_L,R_R$ into $Q$ with their priorities\;
    }
}

}
$\mathcal U\gets\mathcal U\cup Q$;

Build one KD-tree over each $\mathcal C_{y'}$ and one over $\mathcal U$;

\Return{$\{\mathcal C_{y'}\}_{y'\in\mathcal Y},\mathcal U$};
\end{algorithm}

\paragraph{Counterfactual guarantees.}
At termination, the certified and unresolved regions form a partition of the admissible domain. For each class $y'$, define
\[
\mathcal I_{y'} = 
\bigcup_{R\in\mathcal C_{y'}}(R\cap\mathcal X),
\qquad
\mathcal A = 
\bigcup_{R\in\mathcal U}(R\cap\mathcal X).
\]
Sound certification and domain-preserving refinement imply
\[
\mathcal I_{y'}
\subseteq
\{x\in\mathcal X:\mathcal T(x)=y'\}
\subseteq
\mathcal I_{y'}\cup\mathcal A.
\]
Hence, $\mathcal I_{y'}$ is a certified inner approximation of the target-class decision region, while $\mathcal I_{y'}\cup\mathcal A$ is an outer approximation. If $\mathcal U=\varnothing$, CF-Maps(B) recovers an exact representation.

This directly provides a certificate for counterfactual queries. For a query $x$, let
\[
U(x)=d_p(x,\mathcal I_{y'})
\qquad\text{and}\qquad
L(x)=d_p(x,\mathcal A),
\]
where $U(x)$ is the distance to the nearest certified target region and $L(x)$ is the distance to the nearest unresolved region. The former gives a valid counterfactual and therefore an upper bound on the globally optimal distance $d^\star(x)$. Any target region not yet certified must lie inside $\mathcal A$. Consequently,
\[
\min\{U(x),L(x)\}
\le d^\star(x)\le U(x).
\]
In particular, if $U(x)\le L(x)$, the returned counterfactual is certified globally optimal. Otherwise,
\[
0\le U(x)-d^\star(x)\le U(x)-L(x),
\]
so $U(x)-L(x)$ provides an additive certificate on its possible suboptimality. CF-Maps(B) therefore always returns a valid counterfactual whenever a certified target region is available, while explicitly quantifying what may have been lost by stopping preprocessing early.

\begin{figure}
    \centering
    \includegraphics[width=0.8\linewidth]{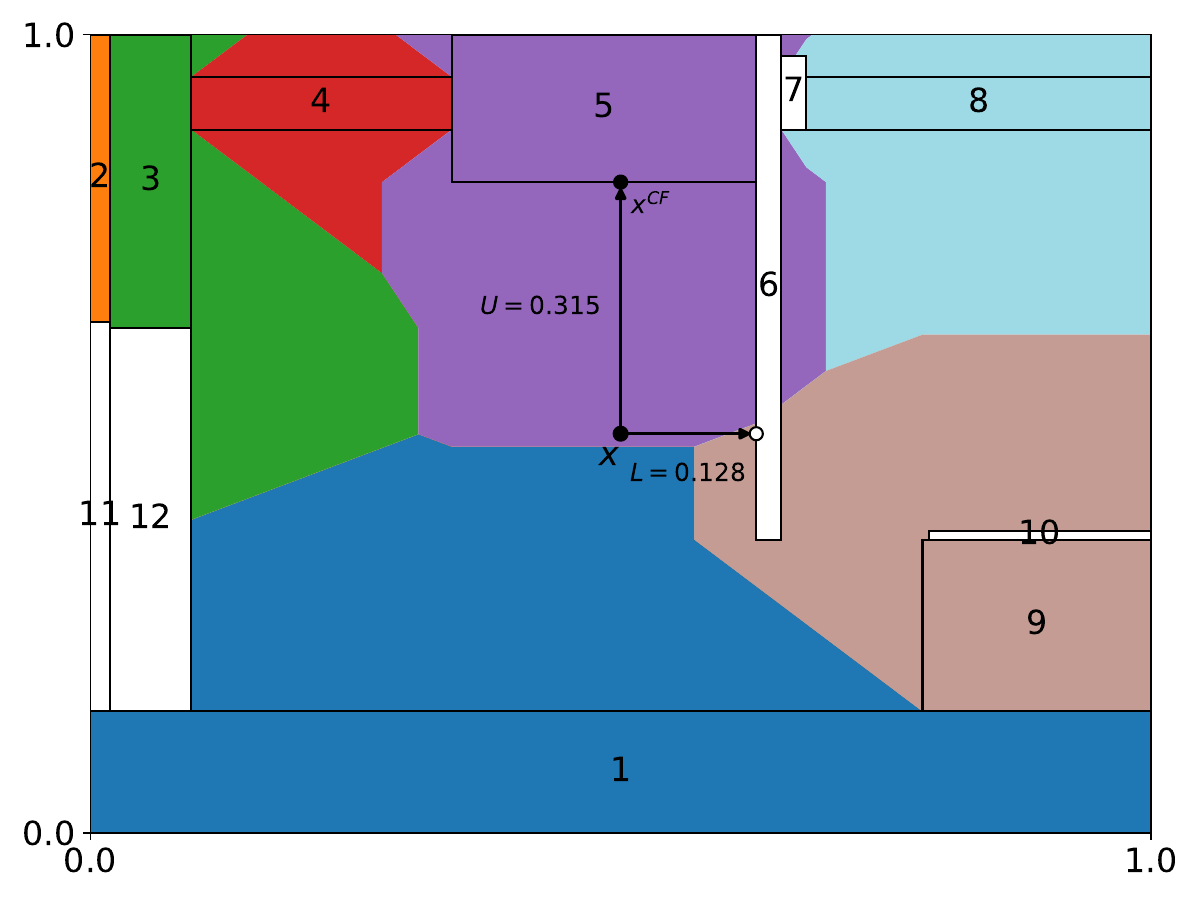}
    \caption{Illustration of CF-Maps(B) under partial certification. Colored areas are assigned to the nearest certified target rectangle, while white rectangles remain uncertified. For query \(x\), the distance to the returned certified counterfactual \(x^{CF}\) provides an upper bound \(U\). The distance to the nearest uncertified rectangle provides \(L\). In this example, $L \leq U$, so $L$ is a lower bound on the optimal counterfactual distance. Thus, the optimal counterfactual distance \(d^\star(x)\) satisfies \(L\leq d^\star(x)\leq U\).
    }
    \label{fig:cfmapsb-example}
\end{figure}

\paragraph{Computational considerations.}
The budget $N$ controls the number of certified regions, but is not a direct bound on runtime: several unsuccessful certification attempts and refinements may be required before producing a new certified region. In the worst case, preprocessing can therefore still approach exhaustive partition extraction. The purpose of CF-Maps(B) is instead to obtain useful certified portions of the map early and to stop refinement before constructing the complete partition.

\subsection{Additional results on large and deep ensembles}\label{app:nps-method-comparison}

Table \ref{tab:cf_maps_b_detailed} compares CF-Maps(B) with the original CF-Maps and with the two best-performing heuristics, namely LIRE and EECE, in terms of preprocessing time, mean query time, and distance ratio (relative to the optimal distance), across all considered datasets, for random forests with large numbers of trees or deep trees. It shows that while CF-Maps(B) exhibits substantially smaller preprocessing times than the original CF-Maps on the most challenging datasets, in practice, it still retrieves optimal counterfactuals in all considered setups. Moreover, in all cases, the lower bound \(\min\{U,L\}\) (not reported in the table) and upper bound \(U\) provided by CF-Maps(B) coincide, indicating that all counterfactuals are also certified optimal. The considered heuristics provide comparatively lower-quality counterfactuals, as evidenced by distance ratios up to $2.82$ for LIRE and $1.39$ for EECE, while EECE also does not scale very well with ensemble size, with mean query time (respectively, preprocessing time) reaching more than $6.5$ seconds (respectively, $8011$ seconds) on the FICO dataset with 1000 depth-10 trees.

\begin{table*}[t]
\centering
\small
\setlength{\tabcolsep}{3pt}
\caption{Comparison of CF-Maps, CF-Maps(B) with $N=10^6$, LIRE, and EECE. $T_P$ is preprocessing time, $T_Q$ is mean query time, and \(D\) is the distance ratio relative to the optimal distance. CF-Maps preprocessing includes BA extraction and KD-tree construction; LIRE and EECE preprocessing are their explainer construction time. Distances are averaged only over valid, domain-feasible counterfactuals. Values are averaged over five seeds.}
\label{tab:cf_maps_b_detailed}\resizebox{\textwidth}{!}{%
\begin{tabular}{llccccccccccccc}
\toprule
Dataset & \multicolumn{2}{c}{Forest} & \multicolumn{3}{c}{CF-Maps} & \multicolumn{3}{c}{CF-Maps(B) ($N=10^6$)} & \multicolumn{3}{c}{LIRE} & \multicolumn{3}{c}{EECE} \\
\cmidrule(lr){4-6}\cmidrule(lr){7-9}\cmidrule(lr){10-12}\cmidrule(lr){13-15}
 & Trees & Depth & $T_P$[s] & $T_Q$[ms] & $D$ & $T_P$[s] & $T_Q$[ms] & $D$ & $T_P$[s] & $T_Q$[ms] & $D$ & $T_P$[s] & $T_Q$[ms] & $D$ \\
\midrule
BC & 100 & 8 & 532.38 & 32.0 & 1.00 & 18.06 & 18.4 & 1.00 & 5.82 & 1.10 & 1.67 & 5.98 & 37.0 & 1.17 \\
\cline{2-15}
 & 100 & 10 & 595.31 & 37.5 & 1.00 & 18.04 & 17.4 & 1.00 & 5.82 & 1.10 & 1.68 & 5.98 & 38.9 & 1.16 \\
\cline{2-15}
 & 100 & 12 & 537.58 & 31.9 & 1.00 & 17.77 & 17.6 & 1.00 & 5.76 & 1.10 & 1.68 & 5.92 & 37.1 & 1.16 \\
\cline{2-15}
 & 200 & 10 & 1009.62 & 37.0 & 1.00 & 30.14 & 19.0 & 1.00 & 11.35 & 1.10 & 1.69 & 11.65 & 75.5 & 1.15 \\
\cline{2-15}
 & 500 & 10 & 2493.35 & 36.4 & 1.00 & 75.69 & 21.6 & 1.00 & 20.90 & 1.20 & 1.71 & 21.73 & 161.7 & 1.12 \\
\cline{2-15}
 & 1000 & 10 & 6633.01 & 34.4 & 1.00 & 180.09 & 23.1 & 1.00 & 37.47 & 1.20 & 1.75 & 39.64 & 384.1 & 1.11 \\
\midrule
CP & 100 & 8 & 0.21 & 0.9 & 1.00 & 0.26 & 1.3 & 1.00 & 2.88 & 0.8 & 1.04 & 3.66 & 194.8 & 1.03 \\
\cline{2-15}
 & 100 & 10 & 0.25 & 1.0 & 1.00 & 0.29 & 1.3 & 1.00 & 2.93 & 0.8 & 1.03 & 3.89 & 237.2 & 1.02 \\
\cline{2-15}
 & 100 & 12 & 0.25 & 0.9 & 1.00 & 0.28 & 1.2 & 1.00 & 2.97 & 0.8 & 1.03 & 3.97 & 241.0 & 1.02 \\
\cline{2-15}
 & 200 & 10 & 0.38 & 1.0 & 1.00 & 0.46 & 1.2 & 1.00 & 5.65 & 0.8 & 1.03 & 7.93 & 495.0 & 1.02 \\
\cline{2-15}
 & 500 & 10 & 0.81 & 1.0 & 1.00 & 1.11 & 1.2 & 1.00 & 10.66 & 0.9 & 1.04 & 29.89 & 1405.2 & 1.02 \\
\cline{2-15}
 & 1000 & 10 & 1.60 & 1.0 & 1.00 & 2.30 & 1.3 & 1.00 & 19.19 & 0.9 & 1.03 & 130.43 & 3254.4 & 1.02 \\
\midrule
FI & 100 & 8 & 1.07 & 7.1 & 1.00 & 3.17 & 13.0 & 1.00 & 48.99 & 7.8 & 1.24 & 54.55 & 113.2 & 1.23 \\
\cline{2-15}
 & 100 & 10 & 1.74 & 7.6 & 1.00 & 5.12 & 25.6 & 1.00 & 48.91 & 7.9 & 1.21 & 96.76 & 222.6 & 1.21 \\
\cline{2-15}
 & 100 & 12 & 2.42 & 7.1 & 1.00 & 6.30 & 11.1 & 1.00 & 49.56 & 7.9 & 1.16 & 246.47 & 408.1 & 1.16 \\
\cline{2-15}
 & 200 & 10 & 2.48 & 7.3 & 1.00 & 7.29 & 12.7 & 1.00 & 95.20 & 7.9 & 1.21 & 331.41 & 602.9 & 1.20 \\
\cline{2-15}
 & 500 & 10 & 5.62 & 7.1 & 1.00 & 22.84 & 13.0 & 1.00 & 177.67 & 8.4 & 1.21 & 2213.31 & 2151.2 & 1.20 \\
\cline{2-15}
 & 1000 & 10 & 10.60 & 7.4 & 1.00 & 42.71 & 13.5 & 1.00 & 326.46 & 8.5 & 1.22 & 8011.11 & 6675.8 & 1.20 \\
\midrule
PD & 100 & 8 & 459.01 & 13.3 & 1.00 & 24.32 & 7.3 & 1.00 & 14.03 & 1.6 & 2.48 & 14.44 & 56.6 & 1.10 \\
\cline{2-15}
 & 100 & 10 & 520.53 & 14.0 & 1.00 & 32.72 & 6.7 & 1.00 & 14.04 & 1.6 & 2.62 & 14.66 & 62.7 & 1.10 \\
\cline{2-15}
 & 100 & 12 & 604.85 & 15.2 & 1.00 & 35.40 & 6.5 & 1.00 & 14.04 & 1.6 & 2.78 & 14.74 & 64.2 & 1.10 \\
\cline{2-15}
 & 200 & 10 & 567.35 & 13.6 & 1.00 & 50.93 & 6.5 & 1.00 & 27.21 & 1.7 & 2.73 & 28.79 & 130.0 & 1.07 \\
\cline{2-15}
 & 500 & 10 & 970.36 & 13.7 & 1.00 & 133.59 & 6.8 & 1.00 & 52.19 & 1.8 & 2.80 & 61.98 & 466.1 & 1.06 \\
\cline{2-15}
 & 1000 & 10 & 2198.57 & 13.2 & 1.00 & 302.58 & 7.1 & 1.00 & 90.49 & 1.8 & 2.82 & 130.28 & 1494.2 & 1.05 \\
\midrule
SE & 100 & 8 & 102.79 & 7.9 & 1.00 & 2.44 & 3.1 & 1.00 & 3.38 & 0.5 & 2.35 & 3.48 & 38.3 & 1.39 \\
\cline{2-15}
 & 100 & 10 & 103.82 & 7.9 & 1.00 & 2.37 & 3.0 & 1.00 & 3.33 & 0.5 & 2.36 & 3.45 & 38.8 & 1.37 \\
\cline{2-15}
 & 100 & 12 & 107.31 & 8.2 & 1.00 & 2.36 & 2.9 & 1.00 & 3.36 & 0.5 & 2.36 & 3.48 & 45.1 & 1.37 \\
\cline{2-15}
 & 200 & 10 & 142.89 & 8.1 & 1.00 & 3.99 & 3.2 & 1.00 & 6.50 & 0.4 & 2.26 & 6.71 & 73.5 & 1.35 \\
\cline{2-15}
 & 500 & 10 & 197.78 & 8.6 & 1.00 & 9.53 & 3.6 & 1.00 & 12.65 & 0.5 & 2.34 & 13.27 & 177.7 & 1.26 \\
\cline{2-15}
 & 1000 & 10 & 322.11 & 9.2 & 1.00 & 21.70 & 3.6 & 1.00 & 22.74 & 0.5 & 2.38 & 24.22 & 393.8 & 1.21 \\
\bottomrule
\end{tabular}
}
\label{tab:nps-method-comparison}
\end{table*}

\subsection{CF-Maps(B) Budget Sensitivity Analysis}
\label{app:budget_ablation}

Table~\ref{tab:nps-obj7-budgets} studies the effect of the certification budget $N\in\{10^4,10^5,10^6\}$. Smaller budgets reduce preprocessing but leave a larger fraction of the decision space unresolved, leading to fewer counterfactuals whose optimality can be certified and, when the certificate is not tight, larger optimality gaps. 
This effect is most visible on BC and PD: at $N=10^4$, only a small fraction of queries are certified optimal, whereas increasing the budget progressively improves both counterfactual distance and the certificate. In contrast, simpler partitions such as those obtained on CP are already fully resolved at the smallest budget.

At $N=10^6$, all queries are certified optimal across the tested configurations, with zero certified gap. This motivates the choice $N=10^6$ used in the large-scale experiments of Section~\ref{sec:results_CF_Maps_B}. Importantly, this behavior is empirical rather than guaranteed: the budget required for exact certification depends on the geometry and complexity of the ensemble's decision regions.
Finally, we emphasize that even when optimality is not achieved, the ability of CF-Maps(B) to provide a certified bound on suboptimality remains a key advantage over heuristic approaches.

\begin{table*}[t]
\centering
\small
\setlength{\tabcolsep}{2.5pt}
\caption{Comparison of the three CF-Maps(B) certification budgets. $T_P$ is preprocessing time (BA extraction plus KD-tree construction) in seconds, $T_Q$ is mean query time in milliseconds, \(D\) is the distance ratio relative to the optimal distance obtained by CF-Maps, \%OPT is the percentage of queries certified optimal ($U\leq L$ up to numerical tolerance), and Gap is the mean difference between the certified upper and lower bounds, normalized by the optimal distance: $\mathrm{Gap} = \mathrm{mean}\left[\max\{0,U-L\}/U_{\mathrm{opt}}\right]$. A smaller Gap indicates a tighter optimality guarantee.}
\resizebox{\textwidth}{!}{%
\begin{tabular}{llcccccccccccccccc}
\toprule
Dataset & \multicolumn{2}{c}{Forest} & \multicolumn{5}{c}{CF-Maps(B) ($N=10^4$)} & \multicolumn{5}{c}{CF-Maps(B) ($N=10^5$)} & \multicolumn{5}{c}{CF-Maps(B) ($N=10^6$)} \\
\cmidrule(lr){4-8}\cmidrule(lr){9-13}\cmidrule(lr){14-18}
 & Trees & Depth & $T_P$[s] & $T_Q$[ms] & $D$ & \%OPT & Gap & $T_P$[s] & $T_Q$[ms] & $D$ & \%OPT & Gap & $T_P$[s] & $T_Q$[ms] & $D$ & \%OPT & Gap \\
\midrule
BC & 100 & 8 & 3.96 & 13.4 & 1.10 & 8.0 & 0.31 & 17.28 & 22.6 & 1.00 & 91.8 & 0.01 & 18.06 & 18.4 & 1.00 & 100.0 & 0.00 \\
\cline{2-18}
 & 100 & 10 & 3.89 & 12.4 & 1.11 & 7.8 & 0.352 & 17.17 & 22.3 & 1.01 & 86.0 & 0.02 & 18.04 & 17.4 & 1.00 & 100.0 & 0.00 \\
\cline{2-18}
 & 100 & 12 & 3.72 & 11.7 & 1.11 & 8.6 & 0.353 & 16.61 & 21.8 & 1.01 & 85.6 & 0.02 & 17.77 & 17.6 & 1.00 & 100.0 & 0.00 \\
\cline{2-18}
 & 200 & 10 & 4.73 & 13.0 & 1.12 & 7.4 & 0.347 & 24.27 & 26.3 & 1.01 & 82.2 & 0.03 & 30.14 & 19.0 & 1.00 & 100.0 & 0.00 \\
\cline{2-18}
 & 500 & 10 & 8.70 & 12.9 & 1.11 & 4.6 & 0.343 & 47.67 & 27.5 & 1.02 & 63.8 & 0.06 & 75.69 & 21.6 & 1.00 & 100.0 & 0.00 \\
\cline{2-18}
 & 1000 & 10 & 16.57 & 12.8 & 1.11 & 7.6 & 0.338 & 100.54 & 26.6 & 1.01 & 63.8 & 0.06 & 180.09 & 23.1 & 1.00 & 100.0 & 0.00 \\
\midrule
CP & 100 & 8 & 0.26 & 1.4 & 1.00 & 100.0 & 0.00 & 0.27 & 1.3 & 1.00 & 100.0 & 0.00 & 0.26 & 1.3 & 1.00 & 100.0 & 0.00 \\
\cline{2-18}
 & 100 & 10 & 0.29 & 1.3 & 1.00 & 100.0 & 0.00 & 0.30 & 1.3 & 1.00 & 100.0 & 0.00 & 0.29 & 1.3 & 1.00 & 100.0 & 0.00 \\
\cline{2-18}
 & 100 & 12 & 0.29 & 1.2 & 1.00 & 100.0 & 0.00 & 0.28 & 1.2 & 1.00 & 100.0 & 0.00 & 0.28 & 1.2 & 1.00 & 100.0 & 0.00 \\
\cline{2-18}
 & 200 & 10 & 0.46 & 1.2 & 1.00 & 100.0 & 0.00 & 0.46 & 1.2 & 1.00 & 100.0 & 0.00 & 0.46 & 1.2 & 1.00 & 100.0 & 0.00 \\
\cline{2-18}
 & 500 & 10 & 1.14 & 1.2 & 1.00 & 100.0 & 0.00 & 1.09 & 1.2 & 1.00 & 100.0 & 0.00 & 1.11 & 1.2 & 1.00 & 100.0 & 0.00 \\
\cline{2-18}
 & 1000 & 10 & 2.27 & 1.3 & 1.00 & 100.0 & 0.00 & 2.28 & 1.4 & 1.00 & 100.0 & 0.00 & 2.30 & 1.3 & 1.00 & 100.0 & 0.00 \\
\midrule
FI & 100 & 8 & 3.21 & 18.4 & 1.00 & 99.8 & 0.01 & 3.19 & 13.0 & 1.00 & 100.0 & 0.00 & 3.17 & 13.0 & 1.00 & 100.0 & 0.00 \\
\cline{2-18}
 & 100 & 10 & 4.72 & 22.0 & 1.01 & 96.6 & 0.03 & 4.60 & 13.4 & 1.00 & 100.0 & 0.00 & 5.12 & 25.6 & 1.00 & 100.0 & 0.00 \\
\cline{2-18}
 & 100 & 12 & 6.30 & 20.4 & 1.00 & 97.4 & 0.02 & 6.42 & 11.2 & 1.00 & 100.0 & 0.00 & 6.30 & 11.1 & 1.00 & 100.0 & 0.00 \\
\cline{2-18}
 & 200 & 10 & 6.84 & 20.6 & 1.00 & 97.6 & 0.02 & 6.96 & 12.8 & 1.00 & 100.0 & 0.00 & 7.29 & 12.7 & 1.00 & 100.0 & 0.00 \\
\cline{2-18}
 & 500 & 10 & 21.61 & 21.2 & 1.00 & 97.8 & 0.02 & 23.38 & 12.8 & 1.00 & 100.0 & 0.00 & 22.84 & 13.0 & 1.00 & 100.0 & 0.00 \\
\cline{2-18}
 & 1000 & 10 & 42.71 & 22.1 & 1.00 & 97.2 & 0.02 & 41.99 & 13.5 & 1.00 & 100.0 & 0.00 & 42.71 & 13.5 & 1.00 & 100.0 & 0.00 \\
\midrule
PD & 100 & 8 & 4.92 & 5.3 & 1.66 & 1.0 & 1.14 & 18.35 & 12.0 & 1.04 & 67.0 & 0.08 & 24.32 & 7.3 & 1.00 & 100.0 & 0.00 \\
\cline{2-18}
 & 100 & 10 & 6.77 & 4.9 & 1.72 & 0.8 & 1.28 & 21.17 & 10.8 & 1.06 & 52.8 & 0.14 & 32.72 & 6.7 & 1.00 & 100.0 & 0.00 \\
\cline{2-18}
 & 100 & 12 & 6.71 & 4.5 & 1.80 & 0.2 & 1.36 & 21.15 & 10.5 & 1.09 & 43.2 & 0.18 & 35.40 & 6.5 & 1.00 & 100.0 & 0.00 \\
\cline{2-18}
 & 200 & 10 & 9.84 & 4.3 & 1.73 & 0.2 & 1.30 & 29.87 & 10.1 & 1.08 & 49.0 & 0.16 & 50.93 & 6.5 & 1.00 & 100.0 & 0.00 \\
\cline{2-18}
 & 500 & 10 & 24.70 & 4.4 & 1.74 & 0.0 & 1.30 & 74.11 & 11.1 & 1.08 & 46.6 & 0.16 & 133.59 & 6.8 & 1.00 & 100.0 & 0.00 \\
\cline{2-18}
 & 1000 & 10 & 50.73 & 4.5 & 1.74 & 0.0 & 1.30 & 148.42 & 10.4 & 1.07 & 44.0 & 0.17 & 302.58 & 7.1 & 1.00 & 100.0 & 0.00 \\
\midrule
SE & 100 & 8 & 1.66 & 4.6 & 1.02 & 31.0 & 0.51 & 2.45 & 3.2 & 1.00 & 100.0 & 0.00 & 2.44 & 3.1 & 1.00 & 100.0 & 0.00 \\
\cline{2-18}
 & 100 & 10 & 1.65 & 4.4 & 1.01 & 28.6 & 0.51 & 2.38 & 3.0 & 1.00 & 100.0 & 0.00 & 2.37 & 3.0 & 1.00 & 100.0 & 0.00 \\
\cline{2-18}
 & 100 & 12 & 1.61 & 4.2 & 1.01 & 28.6 & 0.51 & 2.35 & 3.0 & 1.00 & 100.0 & 0.00 & 2.36 & 2.9 & 1.00 & 100.0 & 0.00 \\
\cline{2-18}
 & 200 & 10 & 2.39 & 4.6 & 1.03 & 31.9 & 0.53 & 3.99 & 3.3 & 1.00 & 100.0 & 0.00 & 3.99 & 3.2 & 1.00 & 100.0 & 0.00 \\
\cline{2-18}
 & 500 & 10 & 4.99 & 4.9 & 1.02 & 32.4 & 0.51 & 9.73 & 3.5 & 1.00 & 100.0 & 0.00 & 9.53 & 3.6 & 1.00 & 100.0 & 0.00 \\
\cline{2-18}
 & 1000 & 10 & 10.13 & 4.8 & 1.03 & 28.1 & 0.53 & 21.33 & 3.5 & 1.00 & 100.0 & 0.00 & 21.70 & 3.6 & 1.00 & 100.0 & 0.00 \\
\bottomrule
\end{tabular}
}
\label{tab:nps-obj7-budgets}
\end{table*}

\FloatBarrier

\end{document}